\documentclass[11pt]{article}
\usepackage{appendix}
\usepackage{amsmath,amssymb,amscd,amsthm,ifthen,color,framed,bm}
\usepackage{graphicx}
\usepackage{rh_defs_21}
\usepackage{listings}

\usepackage{hyperref}
\hypersetup{
    bookmarks=true,         
    unicode=false,          
    pdftoolbar=true,        
    pdfmenubar=true,        
    pdffitwindow=false,     
    pdfstartview={FitH},    
    pdfsubject={Subject},   
    pdfproducer={Producer}, 
    pdfnewwindow=true,      
    colorlinks=true,       
    linkcolor=red,          
    citecolor=blue,        
    filecolor=magenta,      
    urlcolor=cyan           
}
\usepackage{booktabs} 
\usepackage{multirow} 
\usepackage{filecontents}

\usepackage[
backend=bibtex,
url=false,isbn=false,doi=false,
date=year,
hyperref=auto,
style=alphabetic,
natbib=true,
sorting=nty,
firstinits=true,
maxnames=2, maxbibnames=10,
]{biblatex}

\bibliography{bibliography}

\usepackage{pgfplots}
\usepgfplotslibrary{groupplots}
\pgfplotsset{compat=1.18} 

\usetikzlibrary{matrix,arrows,decorations.pathmorphing}

\usepackage{enumitem}

\usetikzlibrary{external}

\usepackage[table]{xcolor}
\usepackage{xparse}
\usepackage{titletoc}
\usepackage{graphicx}
\usepackage[most]{tcolorbox}

\newtheorem{lemma}{Lemma}

\newtheorem{proposition}{Proposition}

\newcommand{\dt}[1]{\frac{{\mathrm{d}}#1}{\mathrm{d}t}}

\newcommand{\rhohat}{\hat\rho}

 \newcommand{\rhotheta}{\rho_{\theta}}

\newcommand{\Ghat}{\widehat G}

\newcommand{\picond}{\pi_\theta(\cdot|x)}

\newcommand{\indicator}[1]{\mathbb{I}\left[#1\right]}

\newcommand{\E}{\mathbb{E}}
\newcommand{\nn}{\nonumber}

\newcommand{\R}{\mathbb{R}}

\newcommand{\M}{M}

\begin{document}

\begin{center}

{
\bf{\LARGE{Asymmetric Prompt Weighting for Reinforcement Learning with Verifiable Rewards}}
}

\vspace*{.2in}

{\large{
\begin{tabular}{cccc}
Reinhard Heckel$^\ast$, Mahdi Soltanolkotabi$^\circ$, Christos Thramboulidis$^\square$
\end{tabular}
}}

\footnotetext{Authors in alphabetical order.}

\vspace*{.05in}

\begin{tabular}{c}
$^\ast$Dept. of Computer Engineering, Technical University of Munich\\
$^\circ$Dept. of Electrical and Computer Engineering, University of Southern California\\
$^\square$Dept. of Electrical and Computer Engineering, University of British Columbia
\end{tabular}

\vspace*{.1in}

\today

\vspace*{.1in}

\end{center}


\begin{abstract}
Reinforcement learning with verifiable rewards has driven recent advances in LLM post-training, in particular for reasoning. Policy optimization algorithms generate a number of responses for a given prompt and then effectively weight the corresponding gradients depending on the rewards. The most popular algorithms including GRPO, DAPO, and RLOO focus on ambiguous prompts, i.e., prompts with intermediate success probability, while downgrading gradients with very easy and very hard prompts. 
In this paper, we consider asymmetric prompt weightings that assign higher weights to prompts with low, or even zero, empirical success probability. 
We find that asymmetric weighting particularly benefits from-scratch RL (as in R1-Zero), where training traverses a wide accuracy range, and less so in post-SFT RL where the model already starts at high accuracy. 
We also provide theory that characterizes prompt weights which minimize the time needed to raise success probability from an initial level to a target accuracy under a fixed update budget. In low-success regimes, where informative responses are rare and response cost dominates, these optimal weights become asymmetric, upweighting low success probabilities and thereby accelerating effective-time convergence.
\end{abstract}

\section{Introduction}

Recent advances for post-training LLMs have substantially improved their reasoning, math, and coding abilities. In particular, reinforcement learning with verifiable rewards (RLVR) that leverages simple--often binary--feedback, such as the correctness of an answer, have driven the recent progress in reasoning performance~\cite{deepseek-ai_DeepSeekR1IncentivizingReasoning_2025,team_KimiK2Open_2025,
openai_OpenAIO1System_2024}. 

Policy optimization algorithms for LLM post-training typically first sample a batch of prompts (e.g., math problems), second generate a set of responses for each prompt (e.g., solutions to math problems), third compute a reward (e.g., 0/1 depending if the solution is correct), and finally update the model based on the normalized rewards. 
The most popular algorithms including GRPO~\cite{shao_DeepSeekMathPushingLimits_2024}, DAPO~\cite{yu_DAPOOpenSourceLLM_2025}, RLOO~\cite{kool_EstimatingGradientsDiscrete_2020,ahmadian_BackBasicsRevisiting_2024} and variants thereof focus on ambiguous prompts, i.e., prompts with intermediate success probability. 

In this paper, we propose asymmetric prompt weightings assigning higher weights to prompts with low success probability. 
Those weightings also assigns non-zero weight even when all completions receive zero reward, whereas GRPO and variants assign zero weight to such gradients, and thus fail to leverage these as training signals.

We perform two sets of experiments: 
First, we consider two from-scratch RL, R1-Zero-like~\cite{deepseek-ai_DeepSeekR1IncentivizingReasoning_2025} setups where we start from a model initially not performing well on reasoning, and we reach high performance through RL. For the first of these, we train the base model Qwen2.5-3B on a countdown task, and for the second we train the base model Llama-3.1 on GSM8K. 
For both setups, the base model starts below 0.02 accuracy and improves to about 0.8, and our asymmetric prompt weighting outperforms symmetric schemes such as GRPO, DAPO, and RLOO. 
Therefore, asymmetric weighting consistently improves performance in the from-scratch regime.

Second, we consider two post-SFT RL setups where the model is already good at reasoning through SFT, and RL is used to further improve performance. Specifically, we train Llama-3.2-3B-instruct on MATH~\cite{hendrycks_MeasuringMathematicalProblem_2021}, and DeepSeek-R1-
Distill-Qwen-1.5B on DAPO-math~\cite{yu_DAPOOpenSourceLLM_2025}. 
In these settings, training begins from an intermediate success rate (0.3 and 0.4) and significantly improves (to 0.5 and 0.55), but we observe no meaningful difference between asymmetric and symmetric advantage weightings. 
In those post-SFT regime, asymmetric weighting yields no additional gain but do not hurt performance.

Finally, we provide theory that characterizes prompt weights which minimize the time needed to raise success probability from an initial level to a target accuracy under a fixed update budget. In low-success regimes, where informative rollouts are rare and rollout cost dominates, these optimal weights become asymmetric, upweighting low success probabilities and thereby accelerating effective-time convergence.


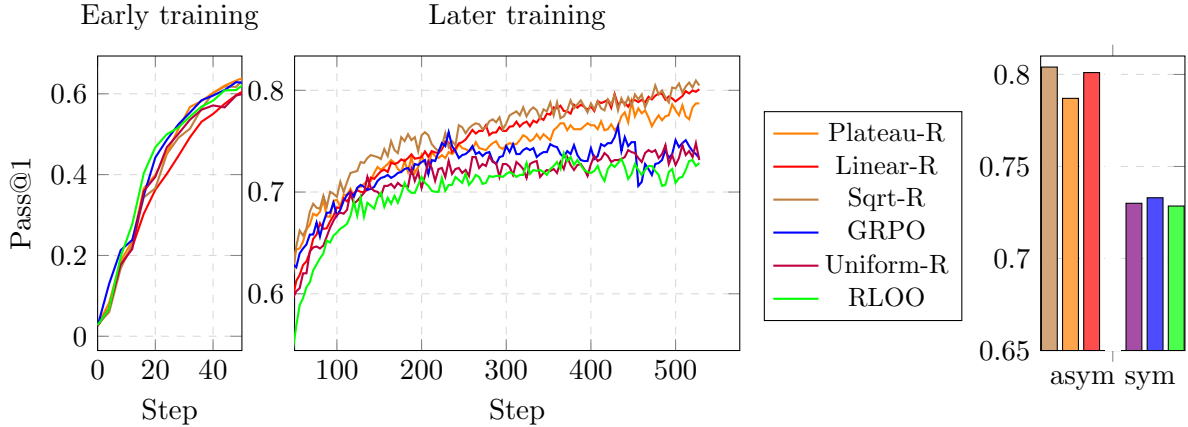
\begin{figure}[tb]
    \centering

\begin{tikzpicture}

    \begin{groupplot}[
        group style={
            group size=3 by 1,
            horizontal sep=0.7cm,
            xlabels at=edge bottom,
            ylabels at=edge left,
        },
        width=7.5cm,
        height=5.5cm,
        xlabel={Step},
        ylabel={Pass@1},
        grid=major,
        grid style={dashed, gray!30},
        legend style={
            at={(1.5,0.1)},
            anchor=south east,
            font=\small,
        },
        table/col sep=comma,
    ]

    \nextgroupplot[xmin=0, xmax=50, title={Early training},width=3.5cm]
        \addplot[red, thick] table[x=Step, y=log_reinforce1] {./fig/tinyzero_comparison.csv};
        \addplot[brown, thick] table[x=Step, y=a_grpo] {./fig/tinyzero_comparison.csv};
        \addplot[orange, thick] table[x=Step, y=a_reinforce] {./fig/tinyzero_comparison.csv};
        \addplot[blue, thick] table[x=Step, y=grpo] {./fig/tinyzero_comparison.csv};
        \addplot[purple, thick] table[x=Step, y=cancel_reinforce] {./fig/tinyzero_comparison.csv};
        \addplot[green, thick] table[x=Step, y=rloo] {./fig/tinyzero_comparison.csv};
    
    \nextgroupplot[xmin=50, title={Later training}, ylabel={}]
        \addplot[orange, thick] table[x=Step, y=a_reinforce] {./fig/tinyzero_comparison.csv};
        \addlegendentry{Plateau-R}
        \addplot[red, thick] table[x=Step, y=log_reinforce1] {./fig/tinyzero_comparison.csv};
        \addlegendentry{Linear-R}
        \addplot[brown, thick] table[x=Step, y=a_grpo] {./fig/tinyzero_comparison.csv};
        \addlegendentry{Sqrt-R}
        \addplot[blue, thick] table[x=Step, y=grpo] {./fig/tinyzero_comparison.csv};
        \addlegendentry{GRPO}
        \addplot[purple, thick] table[x=Step, y=cancel_reinforce] {./fig/tinyzero_comparison.csv};
        \addlegendentry{Uniform-R}
        \addplot[green, thick] table[x=Step, y=rloo_5e-7] {./fig/tinyzero_comparison.csv};
        \addlegendentry{RLOO}

    \nextgroupplot[
        xshift=3.3cm,
        ybar,
        width=3.5cm,
        ylabel={},
        xlabel={},
        symbolic x coords={asym sym},
        xtick=data,
        ymin=0.65,
        ymax=0.81,
        bar width=6pt,
        enlarge x limits=0.4,
        ymajorgrids=true,
        xmajorgrids=false,
    ]
    
    \addplot[fill=brown!70] coordinates {(asym sym,0.804)};
    \addplot[fill=orange!70] coordinates {(asym sym, 0.787)};
    \addplot[fill=red!70] coordinates {(asym sym, 0.801)};

    \addplot[fill=white!70,draw=none] coordinates {(asym sym, 0.801)};

    \addplot[fill=violet!70] coordinates {(asym sym,0.73)};
    \addplot[fill=blue!70] coordinates {(asym sym, 0.733)};
    \addplot[fill=green!70] coordinates {(asym sym, 0.7285)};

    \end{groupplot}
\end{tikzpicture}
\vspace{-0.2cm}
    \caption{From-scratch RL: TinyZero. Test error (i.e., fraction of correctly solved problems, Pass@1) during reinforcement learning. 
    Starting from a score of $0.025$
    the algorithms based on asymmetric weighting (Linear-R, Plateau-R, and Sqrt-R) climb to about $0.8$ while the symmetric weightings (RLOO, Uniform-R, and GRPO) only reach around $0.74$. 
    }
    \label{fig:tinyzerocomparison}
\end{figure}

\subsection{Related work}
There are a number of works that counteract GRPO's weak signal on hard prompts, or better leverage gradients from zero-reward rollouts. 

\citet{team_KimiK15Scaling_2025} used a prioritized sampling strategy where samples are reweighted by $1-\hat \rho$, which effectively makes harder examples more likely to be sampled. \citet{thrampoulidis_AdvantageShapingSurrogate_2025} interprets this through surrogate reward maximization (cf. Section~\ref{sec:surrogate reward perspective}), thus connecting it to GRPO-modifications targeting Pass@K maximization \citep{chen_PasskTrainingAdaptively_2025,mahdavi_AccuracyPolicyGradient_2025}.  While these forms of reweighting hard-examples have an effect analogous to an asymmetric prompt weighting,  prompts corresponding to $\hat \rho=0$ are still set to zero. In contrast, our work directly modifies the advantage estimators to ensure that prompts in the zero-success regime provide a boosted gradient signal. 

\citet{le_NoPromptLeft_2025} addresses zero-variance prompts by entropy-guided advantage shaping at the completion/token level, whereas we introduce prompt-level reweighting based on the success rate. 
\citet{feng_DontWasteMistakes_2025} reshapes learning based on how confident the model was at each mistake, whereas we reshape learning based on how hard the prompt currently is.
\citet{hong_APOEnhancingReasoning_2025}'s `asymmetric policy optimization' for multi-modal LLMs splits responses into positive and negative groups and applies different KL-based regularization schemes, while we adjust advantages. 
CoDaPO~\cite{zhou_CoDaPOConfidenceDifficultyAdaptive_2025} shares our motivation of counteracting GRPO's weak signal on hard prompts, but its difficulty reweighting is only indirectly tied to average performance and is applied jointly with an explicit confidence term, whereas our Linear-R uses a simple prompt-only weight. 
\citet{wang_ASPOAsymmetricImportance_2025} proposed an asymmetric token-level reweighting via policy ratios and advantage sign, whereas we study prompt-level reweighting as a function of success probability. 

An interesting recent result \cite{davis_WhatObjectiveReasoning_2025} shows that different weighting schemes can be interpreted as optimizing different surrogate objectives. Concurrently, \cite{thrampoulidis_AdvantageShapingSurrogate_2025} specifies an explicit recipe for deriving policy-gradient algorithms as weightings of RLOO starting from a surrogate reward. However, both perspectives leave open the practical question of which surrogate (equivalently, which weighting) should be preferred in different regimes. We address this directly: from a practitioner's viewpoint, we clearly demonstrate settings in which different weightings beat the dominant GRPO baseline; from a theory viewpoint, we use a policy-dynamics perspective to provide guidance on selecting weights that account for the prompt-difficulty distribution, showing in particular that asymmetric weights are most beneficial in the rare-success regime.  Finally, while our Linear-R weighting is partly inspired by the log surrogate advocated by \citet{davis_WhatObjectiveReasoning_2025}, their proposed implementation via rejection sampling differs from Linear-R and is empirically less competitive (see Appendix ~\ref{sec:graddir} for a comparison).

\section{Problem setup and background}

Given pairs $(x,a)$ of problem $x$ (in form of a prompt) and (typically) a reference solution/answer $a$, our goal is to maximize the expected reward, 
\begin{align}
\EX[(x,a)]{
\EX[y \sim \pi_\theta(\cdot|x)]{r(x,y,a)}
},
\end{align}
where we take expectation over problems and associated reference solutions, $\pi_\theta$ is the policy (an LLM) parameterized by $\theta$ that we optimize over, $y$ is a response of the LLM to the prompt $x$, and $r$ is a  reward for the response that is computed based on the reference answer $a$. 

For example for mathematical reasoning, the reward checks whether the suggested solution in the completion $y$ is consistent with the reference solution $a$, and for code generation, the reward checks whether a suite of tests $a$ passes. 

The most commonly used policy optimization algorithms for LLMs for RL with verifiable rewards follow the recipe:
\begin{enumerate}
\item Sample a batch of prompts $x_1,\ldots,x_B$ from a dataset.
\item For each prompt $x_j$, the current model $\theta_{\mathrm{old}}$ generates $\M$ responses (also called completions, rollouts, trajectories).
\item Compute a binary reward for each of the responses, for example with a verifier that tests whether a problem specified by $x_j$ has been solved correctly or not.
\item Update the model based on normalized rewards, called advantages, with one gradient step (on-policy training) or several gradient steps (off-policy training).
\end{enumerate}
Policy optimization algorithms such as PPO~\cite{schulman_ProximalPolicyOptimization_2017}, REINFORCE, RLOO, and GRPO differ primarily in: 
\begin{itemize}
\item \textbf{Advantage computation:} How rewards are normalized.
\item \textbf{Loss aggregation:} Sample-average (all rollouts equal, e.g., GRPO), prompt-average (each prompt equal, e.g., DAPO), or token-average.
\item \textbf{Off-policy handling:} On-policy training uses each batch for one update; reusing rollouts for multiple updates becomes off-policy. Algorithms address this mismatch via importance weighting/clipping/KL-penalties, or ignore it.
\end{itemize}
In this paper we focus on an on-policy setup, where each update is a single policy-gradient step based on per-prompt samples and their advantages. 

As discussed below, several widely used policy-optimization methods for LLMs can be interpreted as inducing prompt-level weightings. 
For binary rewards, methods such as GRPO and RLOO yield weights that only depend on the average reward $\hat \rho_x$ of the responses for a prompt $x$ and are symmetric in that prompts that are easy (low $\hat \rho_x$) or  hard (high $\hat \rho_x$) under the current policy receive smaller weight than prompts with intermediate success probability. 

Motivated by the observation that learning can be bottlenecked by hard prompts, we propose and study \textbf{asymmetric} prompt weights that deliberately assign higher weight to poorly-performing prompts (small $\hat \rho_x$). We find  asymmetric reweighting improves learning in from-scratch RL setups. 

\subsection{Policy optimization for verifiable rewards}
\label{sec:policyopt}

We consider rewards bounded in $[0,1]$ for convenience. For a prompt $x$, let $y \sim \pi_\theta(\cdot | x)$ be the response and define the expected reward 
\begin{align}\label{eq:rho_x theta}
\rho_x(\theta) = \EX[y \sim \pi_{\theta}(\cdot | x)]{r(y)},
\end{align}
where, for notational simplicity, we omit the dependence of the reward $r$ on a reference answer $a$ and the prompt $x$. 
The canonical per-prompt policy gradient direction is 
\begin{align*}
\nabla_\theta \rho_x(\theta) 
=
\EX[y \sim \pi_{\theta}(\cdot | x)]{r(y) \nabla_\theta \log \pi_\theta(y|x) }. 
\end{align*}
We view policy-gradient updates as aggregating these per-prompt directions with prompt- and policy-dependent scalar weights $w_x(\theta)$, i.e.,
\begin{align*}
\nabla_\theta J(\theta)
=
\EX[x]{w_x(\theta) \nabla_\theta \rho_x(\theta) }.
\end{align*}
Different policy-optimization algorithms correspond to different choices of weights, and different finite-sample estimators $\hat d_x(\theta)$ of the per-prompt direction $\nabla_\theta \rho_x$. 
For example:
\begin{itemize}
\item Williams' REINFORCE~\cite{williams_SimpleStatisticalGradientfollowing_1992} uses weight $w_x=1$ and estimates the per-prompt direction $\nabla_\theta \rho_x$ as 
\begin{align*}
\hat d_x(\theta)
=
\frac{1}{\M} \sum_{i=1}^\M r(y_i)\nabla_\theta \log \pi_\theta(y_i|x).
\end{align*}
\item  
REINFORCE leave-one-out (RLOO)~\cite{kool_EstimatingGradientsDiscrete_2020,ahmadian_BackBasicsRevisiting_2024} also uses weight $w_x=1$ but uses a leave-one-out estimate of the per-prompt direction:
\begin{align*}
\hat d_x(\theta)
=
\frac{1}{\M} \sum_{i=1}^\M
\left( r(y_i) - \bar r_{-i}\right)
\nabla_\theta \log \pi_\theta(y_i|x), 
\end{align*}
where 
$\bar r_{-i} = \frac{1}{\M-1} \sum_{j\neq i} r(y_j)$
is the average reward of all completions but completion $y_i$. 
\item GRPO~\cite{shao_DeepSeekMathPushingLimits_2024} uses weight $w_x = 1/\sigma_x$, where $\sigma_x^2$ is the empirical variance of the rewards $r(y_1),\ldots,r(y_\M)$, and estimates the direction as 
\begin{align*}
\hat d_x(\theta)
=
\frac{1}{\M} \sum_{i=1}^\M
\left( r(y_i) - \hat \rho \right)
\nabla_\theta \log \pi_\theta(y_i|x), 
\end{align*}
where $\hat \rho$ is the average of rewards. 
Dr. GRPO~\cite{liu_UnderstandingR1ZeroLikeTraining_2025c} uses the same directional estimate, but weight $w_x=1$. 
\end{itemize}

\subsubsection{Binary rewards}\label{sec:binary rewards}

For binary rewards, the direction estimates simplify. Define 
$
\widehat \nabla_0 = \frac{1}{\M_0} \sum_{i: r_i = 0} \nabla_\theta \log \pi_\theta(y_i|x)$ and 
$\widehat \nabla_1 = \frac{1}{\M_1} \sum_{i: r_i = 1} \nabla_\theta \log \pi_\theta(y_i|x)$ 
as the average of the gradients with reward $0$ and reward $1$. Here, $\M_0$ and $\M_1$ are the number of rewards that are $0$ and $1$, and $r_i:=r(y_i)$.
With this notation, the gradient updates are:
\begin{itemize}
\item REINFORCE:
$
\widehat G_x(\theta)
= w_x \cdot \hat d_x(\theta) 
= 1 \cdot 
\rhohat\, \widehat \nabla_1\,.
$
\item RLOO:
\begin{align*}
\hspace{-0.5cm}\widehat G_x(\theta)
= w_x \cdot \hat d_x(\theta) 
= 1 \cdot \frac{\M}{\M -1} \hat \rho (1-\hat \rho) \left(\widehat \nabla_1 - \widehat \nabla_0 \right).
\end{align*}
\item GRPO:
\begin{align*}
\hspace{-0.5cm}\widehat G_x(\theta)
=
w_x \cdot \hat d_x(\theta) 
= \frac{1}{\sqrt{\hat \rho (1-\hat \rho)}} \cdot  \hat \rho (1-\hat \rho) \left(\widehat \nabla_1 - \widehat \nabla_0 \right).
\end{align*}
\end{itemize}
For binary RLOO and GRPO, we can view 
$\hat \rho(1-\hat \rho)$ and $\sqrt{\hat \rho(1-\hat \rho)}$ as the effective prompt weight. 
Moreover, the GRPO direction is (up to the constant $\frac{\M}{\M-1}$) the same as for RLOO, only the weight $w_x$ is 
different. 

For both RLOO and GRPO, prompts that are hard given the current policy ($\hat \rho$ close to zero) as well as prompts that are easy ($\hat \rho$ close to one) get de-emphasized. 
This makes sense, given the viewpoint that for hard prompts gradients are unreliable and for easy ones, we see diminishing returns.

\subsubsection{Surrogate reward perspective}
\label{sec:surrogate reward perspective}
\citet{davis_WhatObjectiveReasoning_2025,thrampoulidis_AdvantageShapingSurrogate_2025} noted that the different weights induce different surrogate-loss functions. 
In particular, REINFORCE and RLOO correspond to optimizing the expected reward
$
J(\theta) 
=
\EX[x]{ \rho_x}
$,
while GRPO corresponds to a  surrogate objective of the form 
$
J(\theta) 
=
\EX[x]{ F(\rho_x) 
},
$
for $F(\rho_x)=2 \mathrm{arcsin}(\sqrt{\rho_x})$.
This holds only for binary rewards; for non-binary rewards, no surrogate loss exists.
\citet{thrampoulidis_AdvantageShapingSurrogate_2025} suggests designing binary-reward RLVR updates by choosing differentiable surrogate reward $F$ and up-weighting the RLOO gradient $\omega_x(\rho)\leftarrow F'(\rho)$. Here, we work directly with the weight function $\omega_x(\rho)$, which provides greater flexibility beyond the binary reward setting. More importantly, we propose concrete new weightings and provide detailed empirical and theoretical analysis showing precisely when and why asymmetric weights outperform existing baselines such as GRPO.

\section{Policy optimization with asymmetric prompt weightings for binary verifiable rewards}\label{sec:assymetric}

As discussed, REINFORCE, RLOO, and GRPO leverage prompt weightings that can be viewed as focusing on ambiguous prompts, i.e., they assign high weights to prompt with intermediate success probability, while downgrading easy and hard ones, which can stabilize training, but may also lead to stagnation for hard prompts. 

Here we propose a weighting that focuses on progress by upweighting prompts with low success probabilities, i.e., low values of $\rho_x$ under the current policy. 

As introduced in the previous section, we focus on general RLVR algorithms with gradient updates for a given prompt $x$ consisting of a per-prompt weight $\omega_x$ and finite-sample estimate $\hat d_x(\theta)$ for the per-prompt direction $\nabla_\theta \rho_x$, i.e., 
\begin{align}
\label{eq:weighted emp}
    \Ghat_x 
    = 
    \omega_x(\rhohat)\cdot 
    \hat d_x(\theta).
\end{align}
We focus on the binary-reward case and take the finite-sample estimate $\hat d_x(\theta)$ for the per-prompt direction as the gradient direction of RLOO and GRPO:
\begin{align}\label{eq:graddir}
    \hat d_x(\theta)
    =
    \rhohat(1-\rhohat)\cdot \left(\widehat \nabla_1-\widehat \nabla_0\right).
\end{align}
Equivalently, written in the conventional advantages-form:
\begin{align}\label{eq:weighted emp advantage form}
    &\Ghat_x = \frac{1}{M}\sum_{i=1}^{\M} A_i\cdot \nabla_\theta\log\pi_\theta(y_i|x), \,\qquad\text{ with } A_i=\omega_x(\rhohat)\cdot(r_i-\rhohat)\,.
\end{align}

To understand the benefit of asymmetric prompt weightings, we consider four  weightings that all upweight prompts with low success probabilities relative to GRPO and RLOO, and with a naming that reflects the effective weights that they assign gradients (see Figure~\ref{fig:gradient_weights}):
\begin{itemize} 
\item \textbf{Linear-REINFORCE (Linear-R)}: We consider weighting $w_x(\rho) = \frac{1}{\rho}$, which focuses on failing prompts, and corresponds to the effective (linear) weight $1-\rho$. 
In terms of advantages, in view of Equation~\eqref{eq:weighted emp advantage form}, this corresponds to  $A_i=\frac{1}{\rhohat}(r_i-\rhohat)$. For wrong responses ($r_i=0$) Linear-R assigns advantage $A_i=-1$; thus, the algorithm's overall gradient is non-zero even if $\rhohat=0$.
\item \textbf{Sqrt-REINFORCE (Sqrt-R)}: 
Consider the weight $w_x(\rho) = \frac{1}{\rho \sqrt{1-\rho}}$, which has the $1/\rho$ weighting dominating for small $\rho$, just like Linear-REINFORCE, and for $\rho$ close to one, behaves like the GRPO weighting. 
The effective weight is $\sqrt{1-\rho}$. Just like Linear-$R$, this assigns non-zero effective weight even if $\rhohat=0$
\item \textbf{Plateau-REINFORCE (Plateau-R):}
Consider the weight 
\begin{align}
w_x(\rho)
= 
\begin{cases}
\frac{1}{2\rho(1-\rho)} & \rho < 1/2 \\
\frac{1}{\sqrt{ \rho(1-\rho)}} & \rho \geq 1/2
\end{cases}.
\end{align}
This gives an effective weight that is equal to that of GRPO for $\rho>1/2$ and constant for smaller $\rho$. 
The weight function above corresponds to advantages $A_i=\omega_x(\rho)\cdot (r_i-\rho)$ which  assign non-zero gradient weights $-1/2$ for prompts with $\rho=0$. 

\item \textbf{Uniform-REINFORCE (Uniform-R):} 
Consider the weight $w_x(\rho) = \frac{1}{\rho (1 - \rho)}$, which leads to an effective weight of $1$. We dub this Uniform-REINFORCE because it has effective weight $1$, since it cancels the factor $\rho (1 -  \rho)$, and thus weighs all prompts equally.
\end{itemize}

The weights assigned to gradients and corresponding advantages (for binary rewards) are summarized in Table \ref{tab:effective_weights_final} and illustrated in Figure~\ref{fig:gradient_weights}. 


\begin{figure*}[tb]
\def\rholist{0.0625,0.125,0.1875,0.25,0.3125,0.375,0.4375,0.5,0.5625,0.625,0.6875,0.75,0.8125,0.875,0.9375}

\pgfplotsset{
  agrpo/.style={thick,brown},
  grpo/.style={thick,densely dotted,blue},
  rloo/.style={thick,dashed,green},
  logreinforce/.style={thick,red},
  logodds/.style={thick,dashdotted,violet},
  effective/.style={thick,orange},
}

\begin{tikzpicture}
\begin{groupplot}[
  group style={group size=3 by 1, horizontal sep=1.5cm},
  width=4.5cm, height=4.5cm,
  grid=both,
  xlabel={$\rho$},
]

\nextgroupplot[
  xmin=0.03125,xmax=0.96875,
  ymin=0,ymax=1.05,
  ylabel={weight},
  domain=0:1, samples=400,
  legend columns=1,
  legend style={
    at={(4.1,0.5)},          
    anchor=west,   
  },
]

\addplot[effective] {ifthenelse(x < 0.5, 0.5, sqrt(x*(1-x)))}; 
\addlegendentry{Plateau-R}

\addplot[logreinforce] {1-x};
\addlegendentry{Linear-R}

\addplot[agrpo]{sqrt(1-x)};
\addlegendentry{Sqrt-R}

\addplot[logodds] {1};
\addlegendentry{Uniform-R}

\addplot[grpo] {sqrt(x*(1-x))};
\addlegendentry{GRPO}

\addplot[rloo] {x*(1-x)};
\addlegendentry{RLOO}

\nextgroupplot[
  xmin=0.03125,xmax=0.96875,
  ymin=0,ymax=16.5,
  ylabel={$A^+(\rho)$},
  samples at={\rholist},
]
\addplot[grpo] {sqrt((1-x)/x)};      
\addplot[rloo] {1-x};               
\addplot[logreinforce] {(1-x)/x};   
\addplot[logodds] {1/x};            
\addplot[effective] {ifthenelse(x < 0.5, 1/(2*x), sqrt((1-x)/x))};  
\addplot[agrpo]{ (1-x)/(x*sqrt(1-x))};

\nextgroupplot[
  xmin=0.03125,xmax=0.96875,
  ymin=-16.5,ymax=0,
  ylabel={$A^-(\rho)$},
  samples at={\rholist},
]
\addplot[grpo] {-sqrt(x/(1-x))};    
\addplot[rloo] {-x};                
\addplot[logreinforce] {-1};        
\addplot[logodds] {-1/(1-x)};       
\addplot[effective] {ifthenelse(x < 0.5, -1/(2*(1-x)), -sqrt(x/(1-x)))};  
\addplot[agrpo]{ -x/(x*sqrt(1-x))};

\end{groupplot}
\end{tikzpicture}
\vspace{-0.2cm}
\caption{
\label{fig:gradient_weights} 
Effective weights $\omega_x(\rho)\cdot \rho(1-\rho)$ assigned to gradients, and advantages assigned to correct/wrong responses for the five considered prompt weightings.  
The range $1/32,\ldots,31/32$ is shown since the number of rollouts $\M$ is typically at most $32$. 
}
\end{figure*}
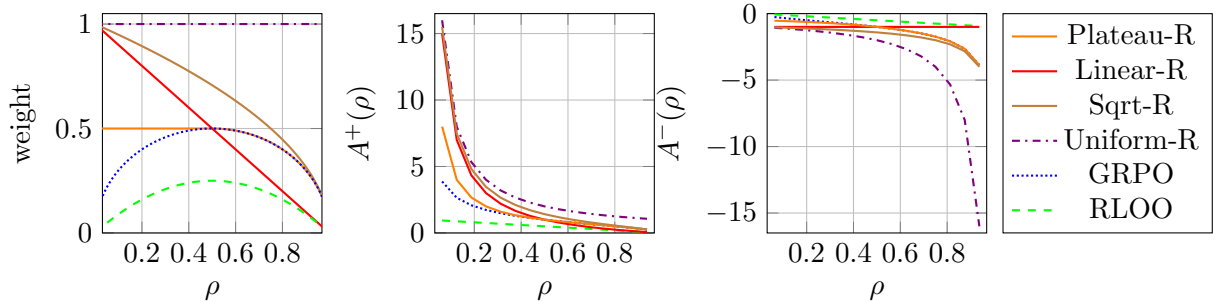

\section{Experiments}

We perform experiments on four different datasets and setups, each covering different regimes. 

We first consider two from-scratch RL reasoning setups. 
Both setups are in a R1-Zero-like~\cite{deepseek-ai_DeepSeekR1IncentivizingReasoning_2025} regime where we start from a base model that initially does not perform well on reasoning, and eventually reach high performance through RL. 
The first setup is TinyZero~\cite{pan_TinyZero_2025} where the task is to solve a countdown and multiplication task. The second setup is GSM8K~\cite{cobbe_TrainingVerifiersSolve_2021}, a math reasoning task, where we start from a base model not trained or finetuned specifically for math. 
In both setups, the base model's average reward starts below 0.02 and advances to around 0.8. We find that our proposed asymmetric prompt weightings (Sqrt-R, Plateau-R, and Linear-R) outperform the non-symmetric ones (GRPO, RLOO, Uniform-R) in both cases.

Second, we consider two post-SFT RL setups where the model is already good at reasoning through SFT, and the goal of RL is to improve performance further. The first uses the DAPO-math dataset~\cite{yu_DAPOOpenSourceLLM_2025}, and the second the MATH dataset~\cite{hendrycks_MeasuringMathematicalProblem_2021}. In the two setups the average reward starts at 0.3 and 0.4 and reach 0.5 and 0.55, respectively. Unlike the from-scratch setting, we observe no notable difference among the advantage estimators GRPO, RLOO, Plateau-R, and Linear-R.

We use the SkyRL codebase~\cite{cao2025skyrl}, and have implemented our advantage estimators withing that framework; experimental details are in the text and Appendix~\ref{app:experimental_details}. 

\subsection{From-scratch RL: TinyZero}
We begin with the TinyZero setup proposed by \citet{pan_TinyZero_2025}, training the Qwen2.5-3B base model to solve countdown and multiplication tasks. An illustrative prompt is:

{\small
\begin{verbatim}
<|im_start|>system
You are a helpful assistant. You first think 
about the reasoning process in your mind 
and then provides the user with the answer.
<|im_end|>
<|im_start|>user
Using the numbers [3,6,25,50,75,100], create 
an equation that equals 952. You can use 
basic arithmetic operations (+, -, *, /) 
and each number can only be used once. Show
your work in <think> </think> tags. And 
return the final answer in <answer> 
</answer> tags, for example 
<answer> (1 + 2) / 3 </answer>.
<|im_end|> <|im_start|>assistant
Let me solve this step by step.
<think>
\end{verbatim}
}

We use a binary outcome reward only. Original TinyZero~\cite{pan_TinyZero_2025} uses a format reward of 0.1 when the format is correct, in addition to an outcome  reward; however, we find this unnecessary to reach good performance. 
We choose this setup because, as noted by~\cite{pan_TinyZero_2025}, it is a minimal setup where one can experience the ``Aha moment'' reported in the R1-paper~\cite{deepseek-ai_DeepSeekR1IncentivizingReasoning_2025}, where the model starts to have a productive thinking process. 
We choose $\M=16$ rollouts per prompt and a batch size of 512 (i.e., 512 distinct prompts, and for each prompt $\M=16$ completions). 

Figure~\ref{fig:tinyzerocomparison} shows the test loss (Pass@1) for different weightings: The model's initial performance is very low (about $0.025$) and  climbs  to $0.8$ for Linear-R and Plateau-R, compared to $0.74$ for the other algorithms RLOO, GRPO, and Uniform-R. The asymmetric prompt weightings Plateau-R and Linear-R perform significantly better (6\% higher Pass@1 score), and this improvement is reproducible: the plot shows three independent runs for Linear-R. 

We hypothesize the following explanation for this behavior: All algorithms perform similar up to step around 150-200, where the score is about $0.74$. From here on, the asymmetric prompt weightings (Linear-R, Plateau-R, Sqrt-R) continue to improve to $0.8$ while the symmetric ones (RLOO, GRPO, and Uniform-R) saturate. 
Figure~\ref{fig:rhohatdistributionTinyZero} shows the distribution of the values of $\hat \rho$ (for Linear-R, though the other distributions are essentially the same at Step 220): at Step 220 there are still examples with very low values for $\hat \rho$. GRPO and RLOO will assign very small weights to the corresponding gradients (see Figure~\ref{fig:gradient_weights}), as opposed to Linear-R, Plateau-R, Sqrt-R which assigns higher weights and thus continues to make progress on those examples, translating in an overall better score. 

Thus, upweighting gradients for low-$\rho$ prompts can translate in performance gains in settings where many prompts remain difficult to solve and thus have low $\rho$. In Section~\ref{sec:theory}, we provide theoretical support for this finding by showing, under a tractable model of learning dynamics, that the Linear-R weighting is optimal in the low-$\rho$ regime.

Finally, recall that our asymmetric weightings (Linear-R, Plateau-R, Sqrt-R) assign non-zero advantages even if all completions have zero reward ($\hat \rho=0$), unlike GRPO and RLOO. We conducted an ablation in which we set the advantage to zero whenever all completions failed; this hurt performance, see the appendix.

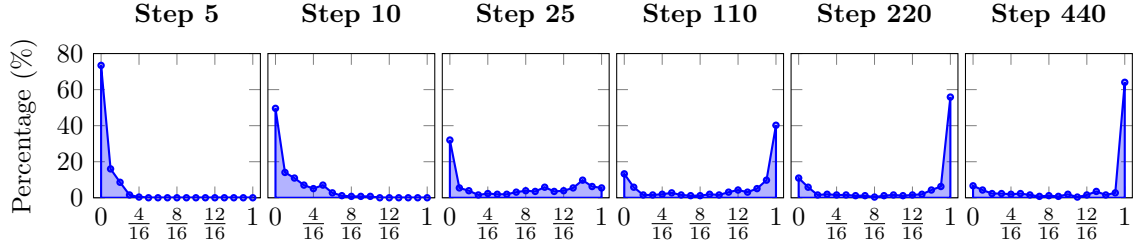
\begin{figure*}[bt]
\begin{center}
\begin{tikzpicture}

\pgfplotstableread{
bin_start step_5 step_10 step_25 step_110 step_220 step_440
0.0000 73.44 49.61 32.03 13.28 10.94 6.64
0.0625 16.02 14.06 5.47 5.86 5.86 4.30
0.1250 8.59 10.94 3.91 1.56 1.56 2.34
0.1875 1.56 7.03 1.56 1.56 1.95 2.34
0.2500 0.39 5.08 2.34 1.95 1.56 1.95
0.3125 0.00 7.03 1.95 2.73 1.56 2.34
0.3750 0.00 2.73 1.95 1.56 1.17 1.56
0.4375 0.00 1.17 3.12 1.17 1.17 0.78
0.5000 0.00 0.78 3.91 1.17 0.39 1.17
0.5625 0.00 0.78 3.52 1.95 1.17 0.78
0.6250 0.00 0.78 5.86 1.56 1.56 1.95
0.6875 0.00 0.00 3.52 3.12 1.17 0.39
0.7500 0.00 0.00 3.91 4.30 1.56 1.56
0.8125 0.00 0.00 5.47 3.12 1.95 3.52
0.8750 0.00 0.00 9.77 5.08 4.30 1.56
0.9375 0.00 0.00 6.25 9.77 6.25 2.73
1.0000 0.00 0.00 5.47 40.23 55.86 64.06
}\datatable

\begin{groupplot}[
    group style={
        group size=6 by 1,
        horizontal sep=0.7cm,
        ylabels at=edge left,
        yticklabels at=edge left,
        horizontal sep=0.1cm,
    },
    width=3.8cm,
    height=3.5cm,
    ymin=0,
    ymax=80,
    ylabel near ticks,
    xlabel near ticks,
    xtick={0,0.25,0.5,0.75,1},
    xticklabels={$0$,$\frac{4}{16}$,$\frac{8}{16}$,$\frac{12}{16}$,$1$},
    xmin=-0.05,
    xmax=1.05,
    tick label style={font=\small},
    label style={font=\small},
    title style={font=\small\bfseries},
]

\nextgroupplot[title={Step 5}, ylabel={Percentage (\%)}]
\addplot[thick, blue, mark=*, mark size=1pt, fill=blue, fill opacity=0.3] table[x=bin_start, y=step_5] {\datatable} \closedcycle;

\nextgroupplot[title={Step 10}]
\addplot[thick, blue, mark=*, mark size=1pt, fill=blue, fill opacity=0.3] table[x=bin_start, y=step_10] {\datatable} \closedcycle;

\nextgroupplot[title={Step 25}]
\addplot[thick, blue, mark=*, mark size=1pt, fill=blue, fill opacity=0.3] table[x=bin_start, y=step_25] {\datatable} \closedcycle;

\nextgroupplot[title={Step 110}]
\addplot[thick, blue, mark=*, mark size=1pt, fill=blue, fill opacity=0.3] table[x=bin_start, y=step_110] {\datatable} \closedcycle;

\nextgroupplot[title={Step 220}]
\addplot[thick, blue, mark=*, mark size=1pt, fill=blue, fill opacity=0.3] table[x=bin_start, y=step_220] {\datatable} \closedcycle;

\nextgroupplot[title={Step 440}]
\addplot[thick, blue, mark=*, mark size=1pt, fill=blue, fill opacity=0.3] table[x=bin_start, y=step_440] {\datatable} \closedcycle;
\end{groupplot}
\end{tikzpicture}
\end{center}
\vspace{-0.55cm}
\caption{
\label{fig:rhohatdistributionTinyZero}
Distribution of the fraction of correct responses out of 16, $\hat{\rho}_x$, for each prompt during a run of Linear-R on TinyZero. At early steps, the distribution is heavily concentrated at low $\hat{\rho}$, reflecting that most prompts are difficult for the base model. Importantly, a substantial number of difficult prompts remain even at later steps (e.g., step 220). Linear-R continues to provide strong gradient signal for these prompts due to its weighting, unlike GRPO.
}
\end{figure*}

\begin{figure}[htb]
    \centering
\begin{tikzpicture}
    \begin{groupplot}[
        group style={
            group size=2 by 1,
            horizontal sep=1cm,
        },
        height=5cm,
        grid=major,
        grid style={dashed, gray!30},
    ]
    
    \nextgroupplot[
        ybar,
        width=5cm,
        ylabel={Pass@1},
        symbolic x coords={B=128, B=256},
        xtick=data,
        ymin=0.65,
        ymax=0.81,
        bar width=6pt,
        enlarge x limits=0.4,
        ymajorgrids=true,
        xmajorgrids=false,
    ]

    \addplot[fill=orange!70] coordinates {(B=128, 0.792) (B=256, 0.783)};

    \addplot[fill=red!70] coordinates {(B=128, 0.753) (B=256, 0.770)};

    \addplot[fill=brown!70] coordinates {(B=128, 0.764) (B=256, 0.762)};

    \addplot[fill=white!70,draw=white] coordinates {(B=128, 0.792) (B=256, 0.783)};

    \addplot[fill=blue!70] coordinates {(B=128, 0.697) (B=256, 0.7392)};
    
    \nextgroupplot[
        width=5.5cm,
        xlabel={Step},
        legend style={
            at={(0.98,0.02)},
            anchor=south east,
            font=\small,
        },
        table/col sep=comma,
    ] 
    
    \addplot[orange, thick] table[x=Step, y=a_reinforce] {./fig/gsm8k_B128.csv};
    \addlegendentry{Plateau-R}

    \addplot[red, thick] table[x=Step, y=log_reinforce] {./fig/gsm8k_B128.csv};
    \addlegendentry{Linear-R}

    \addplot[brown, thick] table[x=Step, y=agrpo] {./fig/gsm8k_B128.csv};
    \addlegendentry{Sqrt-R}

    \addplot[blue, thick] table[x=Step, y=grpo] {./fig/gsm8k_B128.csv};
    \addlegendentry{GRPO}
    
    \end{groupplot}
\end{tikzpicture}
    \caption{
    From-scratch RL: GSM8K.  
    Test error (i.e., fraction of correctly solved problems, Pass@1) during reinforcement learning on the GSM8K benchmark for the Llama-3.1-8B (base) model. Asymmetric weighting (Plateau-R, Linear-R, Sqrt-R) outperform GRPO, demonstrating the benefit of up-weighting difficult prompts.
    }
    \label{fig:gsm8kcomparison}
\end{figure}
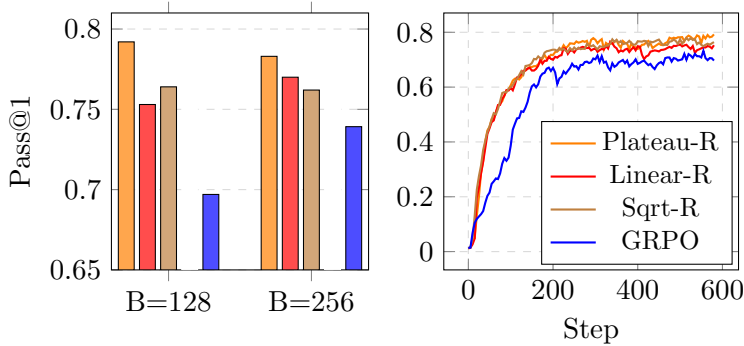

\subsection{From-scratch RL: GSM8K}

Next, we consider the GSM8K dataset. 
We use the Llama-3.1-8B-base model (as opposed to Qwen base models as common in the literature), since Qwen models have been trained on data to improve their math performance, and we wish to understand what is learned through policy optimization algorithms with different weightings/advantages. 
We performed experiments with $\M = 16$ completions, learning rate of 1e-6, and batch sizes of 128 and 256. 
The base model performs, as expected, relatively poorly initially, only achieving a Pass@1 rate of 1.3\%. In this regime, most prompts have no correct completions ($\rho=0$), and dominate the gradient for Plateau-R and Linear-R. We therefore set the advantages of Plateau-R and Linear-R for prompts with success rate $\rho=0$ to $0$ for a warmup period of 50/100 steps for batchsizes 256/128. 

Figure~\ref{fig:gsm8kcomparison} shows that Plateau-R slightly outperforms GRPO for both batch-sizes, demonstrating the benefit of asymmetric prompt weighting. 

This setup follows that considered in the blog post~\cite{schulman_LoRARegret_2025} where the stepsize was tuned for GRPO; and we achieve the same performance with GPRO in our setup.


\subsection{Post-SFT RL: MATH}
\label{sec:postsftmath}

Next, we consider the MATH dataset~\cite{hendrycks_MeasuringMathematicalProblem_2021} consisting of 12k training examples and 500 test examples. Each task is a math problem of varying difficulty levels (1-5) with an answer that is easy to verify, for example a number or fraction. The binary verifier simply compares the solution suggested by the LLM with the provided solution, being lenient with regards to formatting (e.g., \verb|\frac{13}{21}| is the same as $13/21$). 
MATH is frequently used for evaluating GRPO-style algorithms since it provides a challenging set of problems with an automatic and final-answer checker as reward. 

We use the Llama-3.2-3b-instruct model rather than models from the Qwen-series (as common in the literature) because, on MATH, Qwen trained with GRPO-style RLVR improves even with random, incorrect, or format-only rewards~\cite{shao_SpuriousRewardsRethinking_2025}, suggesting the benchmark signal may be confounded by model-specific priors and/or contamination. 
Thus, using Qwen would make it difficult to attribute performance differences to differences in advantage design. 
We choose $M=16$ and our batch-size is 256 (256 distinct math problems and for each $\M=16$ completions). 

From the test-loss (Pass@1) during training for four epochs in Figure~\ref{fig:mathcomparison}(left) it can be seen that the performance of all advantage estimators is very similar. Starting from a score of $0.38$, all variants achieve scores in the range $0.53-56$. See Figure~\ref{fig:rhohatdistribution_math} for the distributions of $\hat \rho$ during training with Linear-R. For TinyZero, we saw that Linear-R improves over the other advantage estimators in the second half of training by improving on examples with low values of $\rho$. For MATH, we don't see such an improvement which is somewhat expected, considering that at step 90 there are almost no examples with close-to-zero (but not equal to zero) values of $\hat \rho$. 

\begin{figure}[tbh]
    \centering
    \begin{tikzpicture}
    \begin{groupplot}[
        group style={
            group size=2 by 1,
            horizontal sep=2cm,
        },
        height=5.5cm,
        width=7.5cm,
        grid=major,
        grid style={dashed, gray!30},
    ]
    
    \nextgroupplot[
        xlabel={Step},
        ylabel={Pass@1},
        legend style={
            at={(0.98,0.02)},
            anchor=south east,
            font=\small,
        },
        table/col sep=comma,
    ]
    
    \addplot[red, thick] table[x=Step, y=log_reinforce] {./fig/math_comparison.csv};
    \addlegendentry{Linear-R}
    \addplot[blue, thick] table[x=Step, y=grpo] {./fig/math_comparison.csv};
    \addlegendentry{GRPO}
    \addplot[purple, thick] table[x=Step, y=c_reinforce] {./fig/math_comparison.csv};
    \addlegendentry{Uniform-R}
    \addplot[green, thick] table[x=Step, y=rloo] {./fig/math_comparison.csv};
    \addlegendentry{RLOO}

    \nextgroupplot[
        xlabel={Step},
        ylabel={Pass@1},
        legend style={
            at={(0.98,0.02)},
            anchor=south east,
            font=\small,
        },
        table/col sep=comma,
    ]
    
    \addplot[orange, thick] table[x=Step, y=a_reinforce] {./fig/dapo_math.csv};
    \addlegendentry{Plateau-R}
    \addplot[red, thick] table[x=Step, y=log_reinforce] {./fig/dapo_math.csv};
    \addlegendentry{Linear-R}
    \addplot[brown, thick] table[x=Step, y=agrpo] {./fig/dapo_math.csv};
    \addlegendentry{Sqrt-R}
    \addplot[blue, thick] table[x=Step, y=grpo] {./fig/dapo_math.csv};
    \addlegendentry{GRPO}

    \end{groupplot}
\end{tikzpicture}
    \caption{Post-SFT RL:
    Test error (i.e., fraction of correctly solved problems) during reinforcement learning on the MATH dataset \textbf{(left)} and on DAPO math \textbf{(right)}. 
    All methods performs similarly well.
    }
    \label{fig:mathcomparison}
    \label{fig:dapo_math_comparison}
\end{figure}
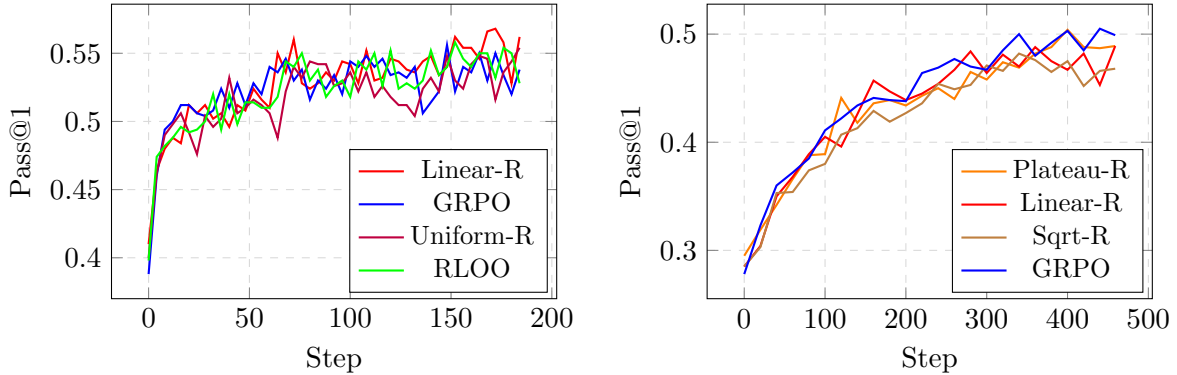

\subsection{Post-SFT RL: DAPO math}
\label{sec:DAPOmath}

Finally, we consider a post-SFT RL setup where we take a model that is already very good at reasoning through supervised finetuning (SFT) on reasoning traces, and improve it further through RL. We choose this setup because models are often SFT'd first and then further improved through RL. 

We follow closely the JustRL setup~\cite{he_JustRLScaling15B_2025} for its simple and clean setup, and so that we can start from well-chosen hyperparameters. 
We train the DeepSeek-R1-Distill-Qwen-1.5B model, and use the DAPO-math dataset (the english subset)~\cite{yu_DAPOOpenSourceLLM_2025} for training. The batch size is 256 prompts, and we take $\M=8$ samples per prompt. The learning rate is 1e-6 and we perform on-policy training without KL penalty nor clipping. 

Figure~\ref{fig:dapo_math_comparison}(right) shows the results: Similar as in Section~\ref{sec:postsftmath} for MATH; all four considered methods perform similarly well, and there is no benefit of up-weighting the difficult prompts (Plateau-R, Linear-R, Sqrt-R), but it also does not hurt performance.

\section{A policy-dynamics perspective to  weighting}\label{sec:theory}

Here, we ask: How do different prompt weights $\omega_x$ shape the dynamics of RLVR. To answer this, we study how the success probability  evolves over training and how different choices of weight functions $\omega_x$ in Equation~\eqref{eq:weighted emp advantage form} induce qualitatively different convergence behaviors. We build on the policy-dynamics analysis of \citet{mroueh_ReinforcementLearningVerifiable_2025}, extending it using the general reweighting framework of Section \ref{sec:assymetric} to cover a broad family of policy optimization algorithms.

We find that in a low-success regime an asymmetric prompt weighting (Sqrt-R) is best, and otherwise GRPO is better; which supports our empirical finding that asymmetric prompt weightings are best in from-scratch RL setups that start from a low-success regime.  

\subsection{Theoretical model}
To isolate the effect of the weighting $\omega_x$ on the learning dynamics, we consider an idealized, population-level
version of policy optimization objectives for RLVR that captures the core ingredients relevant for dynamics: binary rewards,
response-level advantages, and a proximal regularization that prevents overly large policy updates. We consider the following policy optimization per prompt $x$ at each iteration $t$ of the algorithm:
\begin{align}
    \label{eq:start_opt}
\max_{\pi_t(\cdot|x)}\;\; \E_{y\sim \pi_t(\cdot|x)} A_y(\rho_{t-1})  - \beta D_{\text{KL}}\big(\pi_t(\cdot|x)||{\pi_{t-1}(\cdot|x)}\big),
\end{align}
where $A_y(\rho)=\omega_x(\rho)\cdot (r(y)-\rho)$ is the advantage and $\omega_x$ the weight function, as throughout (cf. Equation~\eqref{eq:weighted emp advantage form}). 
Here, $\pi_t$ denotes the policy to be optimized at time $t$ and $\pi_{t-1}$ is the policy of the previous iteration. Unlike the rest of the paper, here we consider the population limit over generated responses $M\rightarrow\infty$; thus, the advantage $A_y(\rho_{t-1})$ of a response $y$ depends on the (population) success rate $\rho_{t-1}$ of the sampling policy $\pi_{t-1}$ at the previous iteration.  
 $D_{\text{KL}}$ is the KL divergence and $\beta>0$ a regularization parameter. 

Optimizing directly over the policy distribution,
rather than over a parametric representation is standard abstraction in the theoretical RL literature, e.g., \citep{rafailov_DirectPreferenceOptimization_2023,mroueh_ReinforcementLearningVerifiable_2025,vojnovic_WhatAlignmentObjective_2025,roux_TaperedOffPolicyREINFORCE_2025}; it allows us to focus on how probability mass shifts between correct and incorrect responses over time.

\subsection{Success-rate dynamics}

As shown in Appendix~\ref{sec:get rho_t ODE} by taking the continuous time limit to the solution of \eqref{eq:start_opt}, the success rate evolves according to the ordinary differential equation (ODE):
\begin{align}\label{eq:ODE}
\dt{\rho_t} = \frac{1}{\beta}\rho_t(1-\rho_t)\cdot \omega(\rho_t)\,.
\end{align}
Thus, the rate of improvement is governed by 
the uncertainty term $\rho_t(1-\rho_t)$ and the weight
$\omega(\rho_t)$. 
Different weights emphasize different regions of the success probabilities.

In practical RLVR, each unit of optimization time corresponds to drawing a finite number of rollouts per prompt. When the success rate $\rho$ is small, the dominant cost is not captured by the optimization clock $t$, but rather by the number of rollouts required before observing a successful response that yields a positive reward.

Under a Bernoulli outcome model, the probability of observing at least one success in $M$ rollouts is $1 - (1-\rho)^M \approx M\rho$ in the low-success regime $\rho \ll 1/M$. Thus, the expected number of rollouts per informative signal scales as $1/(M\rho)$. In this rare-success regime, one unit of population time costs on the order of $g(\rho) = 1/\rho$ rollout units. This motivates measuring progress in an \emph{effective time} variable $\tau$ that upweights time spent at small $\rho$:
\begin{align}\label{eq:effective_time}
\mathrm{d}\tau = g(\rho_t)\,\mathrm{d}t, \qquad g(\rho) = 1/\rho.
\end{align}
Applying this change of variables to Equation~\eqref{eq:ODE}, the success rate $\rho_\tau$ evolves according to the effective-time ODE:
\begin{align}
\label{eq:effective_time_ode}
\frac{\mathrm{d}\rho_\tau}{\mathrm{d}\tau} = \frac{1}{\beta}\,\rho_\tau^2(1-\rho_\tau) \cdot \omega(\rho_\tau).
\end{align}



\begin{figure}[tb]
    \centering
    \begin{tikzpicture}
\begin{groupplot}[
  group style={
    group size=4 by 1,
    horizontal sep=0.1cm,
    ylabels at=edge left,
    yticklabels at=edge left,
  },
  height=5cm,
  width=7cm,
  grid=major,
  grid style={dashed, gray!30},
  ymin=0, ymax=1.05,
  ylabel={$\rho_t/\rho_\tau$},
]

    \nextgroupplot[
        title={$\rho_0=0.1$},
        legend style={
            at={(0.98,0.02)},
            anchor=south east,
            font=\small,
        },
        table/col sep=comma,
        xlabel={time $t$},
        xmin=0,
        xmax=10,
        domain=0:10,
    ]  
    %
    %
    %
\pgfmathsetmacro{\rhoO}{0.1}

\pgfmathsetmacro{\betaRLOO}{1-\rhoO}

\pgfmathsetmacro{\betaLin}{ln(1/\rhoO)}

\pgfmathsetmacro{\zS}{sqrt(1-\rhoO)}
\pgfmathsetmacro{\betaS}{ln((1+\zS)/(1-\zS))}

\pgfmathsetmacro{\aG}{rad(asin(sqrt(\rhoO)))}
\pgfmathsetmacro{\betaG}{pi - 2*\aG}


\addplot[red, thick]
    {1 - (1-\rhoO)*exp(-x/\betaLin)};
\addlegendentry{Linear-R}

\addplot[brown, thick]
    {1 - max(0, \zS - x/(2*\betaS))^2};
\addlegendentry{Sqrt-R}

\addplot[blue, thick, densely dotted]
    {sin(deg(min(pi/2, x/(2*\betaG) + \aG)))^2};
\addlegendentry{GRPO}

\addplot[green, thick, dashed]
    {1/(1 + ((1-\rhoO)/\rhoO)*exp(-x/\betaRLOO))};
\addlegendentry{RLOO}

    \nextgroupplot[
        title={$\rho_0=0.03$},
        legend style={
            at={(0.98,0.02)},
            anchor=south east,
            font=\small,
        },
        table/col sep=comma,
        xlabel={effective time $\tau$},
        domain=0:50, samples=200
    ]  
    %

\pgfmathsetmacro{\rhoO}{0.03}
\pgfmathsetmacro{\betaRLOO}{1-\rhoO}
\pgfmathsetmacro{\betaLin}{ln(1/\rhoO)}
\pgfmathsetmacro{\zS}{sqrt(1-\rhoO)}
\pgfmathsetmacro{\betaS}{ln((1+\zS)/(1-\zS))}
\pgfmathsetmacro{\aG}{rad(asin(sqrt(\rhoO)))}
\pgfmathsetmacro{\betaG}{pi - 2*\aG}

\addplot[red, thick]
{1/(1 + ((1-\rhoO)/\rhoO)*exp(-x/\betaLin))};
\addplot[brown, thick]
{1/(cosh(max(0, 0.5*\betaS - x/(2*\betaS)))^2)};
\pgfmathsetmacro{\zG}{sqrt((1-\rhoO)/\rhoO)}
\addplot[blue, thick, densely dotted]
{1/(1 + max(0, \zG - x/(2*\betaG))^2)};
\addplot[green, thick, dashed]
table[
  x expr=\thisrow{tau}*\betaRLOO,
  y=rho_rloo,
  col sep=comma
]{./fig/rho_effective.csv};
    
    \end{groupplot}
\end{tikzpicture}

    \caption{Comparison of success-probability dynamics $\rho$ governed by the ODE in Eq.~\eqref{eq:ODE} (with $\beta=1$) for RLOO, GRPO, Linear-R and Sqrt-R, with their weights normalized  to satisfy the same total budget constraint Eq. \eqref{eq:budget} (with $B=1$). \textbf{Left:} regular time $t$  with initialization $\rho_0=0.1$. \textbf{Right:} effective time $\tau$  with initialization $\rho_0=0.03$, where $\tau$ accounts for rollout cost in the low-success regime. As predicted by Propostions \ref{prop:single} and \ref{prop:single_effective_time}, GRPO is optimal in regular time, whereas Sqrt-R is optimal in effective time in terms of time required to reach target success rate $\rho^*=1$.
    The right and left plot are obtained by plotting the dynamics stated explicitly in Propositions~\ref{prop:ode_solutions_regular_time} and \ref{prop:closed_form_tau}, using the budget normalization from Lemma~\ref{lem:budget} in the appendix.
    } 
    \label{fig:dynamics_comparison_normalized}
\end{figure}
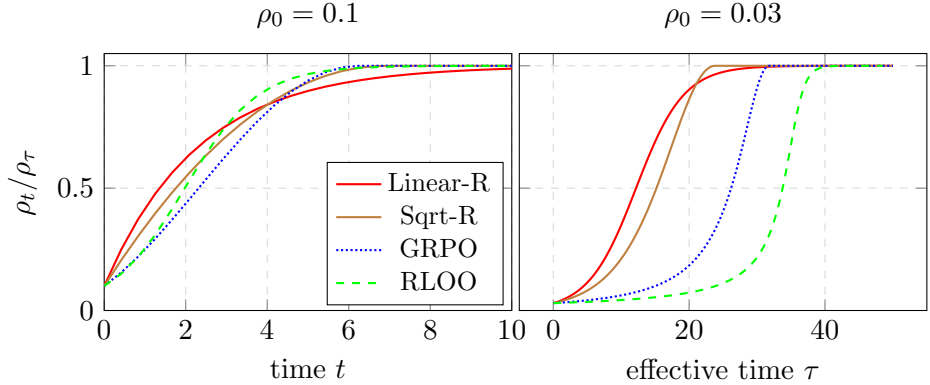

\subsection{Optimal weighting}

Given the dynamics for regular and effective time, we now ask which choice of the weight function $\omega$ leads to the fastest improvement in success probability. 
We find that for regular time GRPO is optimal, while for effective time Sqrt-R is optimal.

\paragraph{Regular-time optimality of GRPO.}
Fix prompt $x$ and let its average success probability $\rho_t$ at iteration $t$ evolve according to the ODE in Equation~\eqref{eq:ODE}. Our goal is to minimize the time required for the success rate to reach a target value $\rho_*$ from initial value $\rho_0$. To exclude trivial solutions, and to ensure a fair normalization across all weight functions, we impose a budget $B$ on the total allowed change of the weights in the interval $[\rho_0,\rho_*]$: 
\begin{align}\label{eq:budget}
\int_{\rho_0}^{\rho_*} \omega(\rho)\,\mathrm{d}\rho \leq B\,,
\end{align}
where we also consider $\omega(\rho)\geq 0$. Without loss of generality, fix $B=1$ onwards.
{Such a budget constraint allows us to fairly compare among algorithms with different weights by normalizing the total ``update budget''. Specifically, Eq. \eqref{eq:budget} bounds the cumulative magnitude of the per-prompt weighting coefficient $\omega_x(\rho)$ used in Equation~\eqref{eq:weighted emp}.
While one could alternatively normalize the optimization using the effective weight $\omega(\rho)\cdot\rho(1-\rho)$ that scales the gradient difference $\widehat\nabla_1-\widehat\nabla_0$ (see Equation \eqref{eq:graddir}), this specific alternative normalization does not qualitatively alter the main conclusion of our analysis: regular-time dynamics favor symmetric weightings, whereas effective-time dynamics in the rare-success regime favor asymmetric weightings.}
\begin{proposition}\label{prop:single}
Suppose the success rate $\rho_t$ of a fixed prompt evolves as in Equation~\eqref{eq:ODE} with initialization $\rho_0$. Let $T(\rho_0,\rho_*;\omega)$ be the time required so that $\rho_t=\rho_*$ for target success rate $\rho_*>\rho_0$. The non-negative weight $\omega:[0,1]\rightarrow\R_{\geq 0}$ that minimizes $T$ subject to the budget constraint~\eqref{eq:budget} with $B=1$ is:
    $\omega_{\text{opt}}(\rho) \propto \frac{1}{\sqrt{\rho(1-\rho)}}, \quad \text{for } \rho \in [\rho_0,\rho_*].
    $
\end{proposition}

Here, time $T$ is the continuous analogue to the number of non-zero gradient updates applied to the model. The proposition gives an optimality interpretation to GRPO's weighting as the weighting minimizing the optimization clock, i.e., the number of updates to reach a target success rate, when all weight functions are normalized to satisfy the same budget constraint. 

Figure~\ref{fig:dynamics_comparison_normalized} (left) visualizes this optimality result, showing the evolution of the success rate for Linear-R, Sqrt-R, RLOO, and GRPO with weights normalized to satisfy the budget constraint for $\rho_*=1$ (see Proposition~\ref{prop:ode_solutions_regular_time} for analytical expressions). As predicted, GRPO reaches the target $\rho_*=1$ first. However, the gain over other algorithms is modest, and asymmetric weights, while not optimal, perform comparably. This is consistent with our Post-SFT experimental findings where all methods achieved similar performance.


\paragraph{Effective time-optimality of Sqrt-R in low-success regimes.}

We now characterize the weight function that minimizes the effective time to reach a target success rate in effective time. {As before, for fair comparison across different weights, we normalize by imposing the same budget constraint~\eqref{eq:budget} on the total accumulated weight.}

\begin{proposition}\label{prop:single_effective_time}
    Suppose the success rate $\rho_t$ of a fixed prompt evolves as in Equation~\eqref{eq:ODE} with initialization $\rho_0 \in (0,1)$ and target $\rho_* > \rho_0$. For $g(\rho)$ given in Equation \eqref{eq:effective_time}, define  effective time-to-target as
    \[
    T_g(\rho_0, \rho_*; \omega) := \int_{\rho_0}^{T} g(\rho_t)\,\mathrm{d}t, \qquad \text{where } \rho_T = \rho_*.
    \]
    Under the budget constraint~\eqref{eq:budget} with $B = 1$, the non-negative weight function minimizing $T_g$ over $[\rho_0, \rho_*]$ is
    \begin{align}
    \label{eq:LR-style-weight}
    \omega_{\mathrm{opt}}(\rho) \propto \frac{1}{\rho\sqrt{1-\rho}}.
    \end{align}
\end{proposition}

    Propositions~\ref{prop:single} and~\ref{prop:single_effective_time} optimize the same population dynamics~\eqref{eq:ODE} but under different notions of time. The optimization time $t$ yields the GRPO weighting 
    whereas the rollout-cost-adjusted effective time~\eqref{eq:effective_time}, which more accurately reflects computational costs when $\rho$ is small, yields the weight~\eqref{eq:LR-style-weight}.
    This is exactly the weight of our Sqrt-R. Also,
    it behaves as Linear-R's $1/\rho$ weight in the low-success regime.

Figure~\ref{fig:dynamics_comparison_normalized} (right) visualizes this optimality under budget normalization, showing that Sqrt-R reaches $\rho_*=1$ significantly faster than all other algorithms in effective time, with Linear-R also substantially outperforming GRPO and RLOO (see Proposition~\ref{prop:closed_form_tau} in the appendix for analytical expressions). 
This distinction is practically relevant: from-scratch RL often operates for a substantial period in a low-success regime where successful rollouts are scarce, making effective time the more appropriate metric. In these settings, both our theory and experiments yield that Sqrt-R and Linear-R outperform the strong GRPO baseline.

\section{Discussion and limitation}
In this paper, we have demonstrated that asymmetric prompt weightings can be beneficial, in particular for from-scratch RLVR.

While our analysis focuses on binary rewards, the proposed weightings and advantage estimators extend naturally to rewards in $[0,1]$ by taking $\hat\rho$ to be the empirical mean reward. Several directions merit further study: asymmetric weighting beyond binary rewards; applications to training agentic systems; and scaling to larger models and datasets. 
Our theory also suggests allocating more rollouts when $\rho$ is close to $0$. While interesting, we were not able to pursue these directions due to compute limitations (each reported run in this paper already costs several hundreds of H100 hours).

\subsection*{Acknowledgements}
RH gratefully acknowledges the computing time made available for the RL experiments in this paper on the high-performance computer at the NHR Center of TU Dresden. This center is jointly supported by the Federal Ministry of Research, Technology and Space of Germany and the state governments participating in the NHR (www.nhr-verein.de/unsere-partner). CT gratefully acknowledges support from the Natural Sciences and Engineering Research Council (NSERC) of Canada.

\printbibliography

@misc{ahmadian_BackBasicsRevisiting_2024,
  title = {Back to {{Basics}}: {{Revisiting REINFORCE Style Optimization}} for {{Learning}} from {{Human Feedback}} in {{LLMs}}},
  author = {Ahmadian, Arash and Cremer, Chris and Gall{\'e}, Matthias and Fadaee, Marzieh and Kreutzer, Julia and Pietquin, Olivier and {\"U}st{\"u}n, Ahmet and Hooker, Sara},
  year = 2024,
  number = {arXiv:2402.14740},
  eprint = {2402.14740},
  publisher = {arXiv},
  archiveprefix = {arXiv}
}

@misc{cao2025skyrl,
  title = {{{SkyRL-v0}}: {{Train}} Real-World Long-Horizon Agents via Reinforcement Learning},
  author = {Cao, Shiyi and Hegde, Sumanth and Li, Dacheng and Griggs, Tyler and Liu, Shu and Tang, Eric and Pan, Jiayi and Wang, Xingyao and Malik, Akshay and Neubig, Graham and Hakhamaneshi, Kourosh and Liaw, Richard and Moritz, Philipp and Zaharia, Matei and Gonzalez, Joseph E. and Stoica, Ion},
  year = 2025
}

@misc{chen_PasskTrainingAdaptively_2025,
  title = {Pass@k {{Training}} for {{Adaptively Balancing Exploration}} and {{Exploitation}} of {{Large Reasoning Models}}},
  author = {Chen, Zhipeng and Qin, Xiaobo and Wu, Youbin and Ling, Yue and Ye, Qinghao and Zhao, Wayne Xin and Shi, Guang},
  year = 2025,
  number = {arXiv:2508.10751},
  eprint = {2508.10751},
  archiveprefix = {arXiv}
}

@misc{cobbe_TrainingVerifiersSolve_2021,
  title = {Training {{Verifiers}} to {{Solve Math Word Problems}}},
  author = {Cobbe, Karl and Kosaraju, Vineet and Bavarian, Mohammad and Chen, Mark and Jun, Heewoo and Kaiser, Lukasz and Plappert, Matthias and Tworek, Jerry and Hilton, Jacob and Nakano, Reiichiro and Hesse, Christopher and Schulman, John},
  year = 2021,
  number = {arXiv:2110.14168},
  eprint = {2110.14168},
  publisher = {arXiv},
  archiveprefix = {arXiv}
}

@misc{davis_WhatObjectiveReasoning_2025,
  title = {What Is the Objective of Reasoning with Reinforcement Learning?},
  author = {Davis, Damek and Recht, Benjamin},
  year = 2025,
  number = {arXiv:2510.13651},
  eprint = {2510.13651},
  publisher = {arXiv},
  archiveprefix = {arXiv}
}

@misc{deepseek-ai_DeepSeekR1IncentivizingReasoning_2025,
  title = {{{DeepSeek-R1}}: {{Incentivizing Reasoning Capability}} in {{LLMs}} via {{Reinforcement Learning}}},
  author = {{DeepSeek-AI} and Guo, Daya and Yang, Dejian and Zhang, Haowei and Song, Junxiao and Zhang, Ruoyu and Xu, Runxin and Zhu, Qihao and Ma, Shirong and Wang, Peiyi and Bi, Xiao and Zhang, Xiaokang and Yu, Xingkai and Wu, Yu and Wu, Z. F. and Gou, Zhibin and Shao, Zhihong and Li, Zhuoshu and Gao, Ziyi and Liu, Aixin and Xue, Bing and Wang, Bingxuan and Wu, Bochao and Feng, Bei and Lu, Chengda and Zhao, Chenggang and Deng, Chengqi and Zhang, Chenyu and Ruan, Chong and Dai, Damai and Chen, Deli and Ji, Dongjie and Li, Erhang and Lin, Fangyun and Dai, Fucong and Luo, Fuli and Hao, Guangbo and Chen, Guanting and Li, Guowei and Zhang, H. and Bao, Han and Xu, Hanwei and Wang, Haocheng and Ding, Honghui and Xin, Huajian and Gao, Huazuo and Qu, Hui and Li, Hui and Guo, Jianzhong and Li, Jiashi and Wang, Jiawei and Chen, Jingchang and Yuan, Jingyang and Qiu, Junjie and Li, Junlong and Cai, J. L. and Ni, Jiaqi and Liang, Jian and Chen, Jin and Dong, Kai and Hu, Kai and Gao, Kaige and Guan, Kang and Huang, Kexin and Yu, Kuai and Wang, Lean and Zhang, Lecong and Zhao, Liang and Wang, Litong and Zhang, Liyue and Xu, Lei and Xia, Leyi and Zhang, Mingchuan and Zhang, Minghua and Tang, Minghui and Li, Meng and Wang, Miaojun and Li, Mingming and Tian, Ning and Huang, Panpan and Zhang, Peng and Wang, Qiancheng and Chen, Qinyu and Du, Qiushi and Ge, Ruiqi and Zhang, Ruisong and Pan, Ruizhe and Wang, Runji and Chen, R. J. and Jin, R. L. and Chen, Ruyi and Lu, Shanghao and Zhou, Shangyan and Chen, Shanhuang and Ye, Shengfeng and Wang, Shiyu and Yu, Shuiping and Zhou, Shunfeng and Pan, Shuting and Li, S. S. and Zhou, Shuang and Wu, Shaoqing and Ye, Shengfeng and Yun, Tao and Pei, Tian and Sun, Tianyu and Wang, T. and Zeng, Wangding and Zhao, Wanjia and Liu, Wen and Liang, Wenfeng and Gao, Wenjun and Yu, Wenqin and Zhang, Wentao and Xiao, W. L. and An, Wei and Liu, Xiaodong and Wang, Xiaohan and Chen, Xiaokang and Nie, Xiaotao and Cheng, Xin and Liu, Xin and Xie, Xin and Liu, Xingchao and Yang, Xinyu and Li, Xinyuan and Su, Xuecheng and Lin, Xuheng and Li, X. Q. and Jin, Xiangyue and Shen, Xiaojin and Chen, Xiaosha and Sun, Xiaowen and Wang, Xiaoxiang and Song, Xinnan and Zhou, Xinyi and Wang, Xianzu and Shan, Xinxia and Li, Y. K. and Wang, Y. Q. and Wei, Y. X. and Zhang, Yang and Xu, Yanhong and Li, Yao and Zhao, Yao and Sun, Yaofeng and Wang, Yaohui and Yu, Yi and Zhang, Yichao and Shi, Yifan and Xiong, Yiliang and He, Ying and Piao, Yishi and Wang, Yisong and Tan, Yixuan and Ma, Yiyang and Liu, Yiyuan and Guo, Yongqiang and Ou, Yuan and Wang, Yuduan and Gong, Yue and Zou, Yuheng and He, Yujia and Xiong, Yunfan and Luo, Yuxiang and You, Yuxiang and Liu, Yuxuan and Zhou, Yuyang and Zhu, Y. X. and Xu, Yanhong and Huang, Yanping and Li, Yaohui and Zheng, Yi and Zhu, Yuchen and Ma, Yunxian and Tang, Ying and Zha, Yukun and Yan, Yuting and Ren, Z. Z. and Ren, Zehui and Sha, Zhangli and Fu, Zhe and Xu, Zhean and Xie, Zhenda and Zhang, Zhengyan and Hao, Zhewen and Ma, Zhicheng and Yan, Zhigang and Wu, Zhiyu and Gu, Zihui and Zhu, Zijia and Liu, Zijun and Li, Zilin and Xie, Ziwei and Song, Ziyang and Pan, Zizheng and Huang, Zhen and Xu, Zhipeng and Zhang, Zhongyu and Zhang, Zhen},
  year = 2025,
  number = {arXiv:2501.12948},
  eprint = {2501.12948},
  publisher = {arXiv},
  archiveprefix = {arXiv}
}

@misc{feng_DontWasteMistakes_2025,
  title = {Don't {{Waste Mistakes}}: {{Leveraging Negative RL-Groups}} via {{Confidence Reweighting}}},
  author = {Feng, Yunzhen and Jain, Parag and Hartshorn, Anthony and Duan, Yaqi and Kempe, Julia},
  year = 2025,
  number = {arXiv:2510.08696},
  eprint = {2510.08696},
  publisher = {arXiv},
  archiveprefix = {arXiv}
}

@misc{he_JustRLScaling15B_2025,
  title = {{{JustRL}}: {{Scaling}} a 1.{{5B LLM}} with a {{Simple RL Recipe}}},
  author = {He, Bingxiang and Qu, Zekai and Liu, Zeyuan and Chen, Yinghao and Zuo, Yuxin and Qian, Cheng and Zhang, Kaiyan and Chen, Weize and Xiao, Chaojun and Cui, Ganqu and Ding, Ning and Liu, Zhiyuan},
  year = 2025,
  number = {arXiv:2512.16649},
  eprint = {2512.16649},
  publisher = {arXiv},
  archiveprefix = {arXiv}
}

@inproceedings{hendrycks_MeasuringMathematicalProblem_2021,
  title = {Measuring {{Mathematical Problem Solving With}} the {{MATH Dataset}}},
  booktitle = {Neural {{Information Processing Systems}} ({{NeurIPS}})},
  author = {Hendrycks, Dan and Burns, Collin and Kadavath, Saurav and Arora, Akul and Basart, Steven and Tang, Eric and Song, Dawn and Steinhardt, Jacob},
  year = 2021,
  eprint = {2103.03874},
  publisher = {arXiv},
  archiveprefix = {arXiv}
}

@misc{hong_APOEnhancingReasoning_2025,
  title = {{{APO}}: {{Enhancing Reasoning Ability}} of {{MLLMs}} via {{Asymmetric Policy Optimization}}},
  author = {Hong, Minjie and Guo, Zirun and Xia, Yan and Wang, Zehan and Zhang, Ziang and Jin, Tao and Zhao, Zhou},
  year = 2025,
  number = {arXiv:2506.21655},
  eprint = {2506.21655},
  publisher = {arXiv},
  archiveprefix = {arXiv}
}

@inproceedings{kool_EstimatingGradientsDiscrete_2020,
  title = {Estimating {{Gradients}} for {{Discrete Random Variables}} by {{Sampling}} without {{Replacement}}},
  booktitle = {International {{Conference}} on {{Learning Representations}} ({{ICLR}})},
  author = {Kool, Wouter and van Hoof, Herke and Welling, Max},
  year = 2020,
  eprint = {2002.06043},
  archiveprefix = {arXiv}
}

@misc{le_NoPromptLeft_2025,
  title = {No {{Prompt Left Behind}}: {{Exploiting Zero-Variance Prompts}} in {{LLM Reinforcement Learning}} via {{Entropy-Guided Advantage Shaping}}},
  author = {Le, Thanh-Long V. and Jeon, Myeongho and Vu, Kim and Lai, Viet and Yang, Eunho},
  year = 2025,
  number = {arXiv:2509.21880},
  eprint = {2509.21880},
  publisher = {arXiv},
  archiveprefix = {arXiv}
}

@misc{liu_UnderstandingR1ZeroLikeTraining_2025c,
  title = {Understanding {{R1-Zero-Like Training}}: {{A Critical Perspective}}},
  author = {Liu, Zichen and Chen, Changyu and Li, Wenjun and Qi, Penghui and Pang, Tianyu and Du, Chao and Lee, Wee Sun and Lin, Min},
  year = 2025,
  number = {arXiv:2503.20783},
  eprint = {2503.20783},
  publisher = {arXiv},
  archiveprefix = {arXiv}
}

@inproceedings{mahdavi_AccuracyPolicyGradient_2025,
  title = {Beyond {{Accuracy}}: {{A Policy Gradient Reweighting Approach}} for {{Pass}}@{{K Maximization}} in {{LLMs}}},
  booktitle = {2nd {{AI}} for {{Math Workshop}} @ {{ICML}} 2025},
  author = {Mahdavi, Sadegh and Li, Muchen and Liu, Kaiwen and Liao, Renjie and Thrampoulidis, Christos},
  year = 2025,
  langid = {english}
}

@misc{mroueh_ReinforcementLearningVerifiable_2025,
  title = {Reinforcement {{Learning}} with {{Verifiable Rewards}}: {{GRPO}}'s {{Effective Loss}}, {{Dynamics}}, and {{Success Amplification}}},
  author = {Mroueh, Youssef},
  year = 2025,
  number = {arXiv:2503.06639},
  eprint = {2503.06639},
  publisher = {arXiv},
  archiveprefix = {arXiv}
}

@misc{openai_OpenAIO1System_2024,
  title = {{{OpenAI}} O1 {{System Card}}},
  author = {OpenAI and Jaech, Aaron and Kalai, Adam and Lerer, Adam and Richardson, Adam and {El-Kishky}, Ahmed and Low, Aiden and Helyar, Alec and Madry, Aleksander and Beutel, Alex and Carney, Alex and Iftimie, Alex and Karpenko, Alex and Passos, Alex Tachard and Neitz, Alexander and Prokofiev, Alexander and Wei, Alexander and Tam, Allison and Bennett, Ally and Kumar, Ananya and Saraiva, Andre and Vallone, Andrea and Duberstein, Andrew and Kondrich, Andrew and Mishchenko, Andrey and Applebaum, Andy and Jiang, Angela and Nair, Ashvin and Zoph, Barret and Ghorbani, Behrooz and Rossen, Ben and Sokolowsky, Benjamin and Barak, Boaz and McGrew, Bob and Minaiev, Borys and Hao, Botao and Baker, Bowen and Houghton, Brandon and McKinzie, Brandon and Eastman, Brydon and Lugaresi, Camillo and Bassin, Cary and Hudson, Cary and Li, Chak Ming and de Bourcy, Charles and Voss, Chelsea and Shen, Chen and Zhang, Chong and Koch, Chris and Orsinger, Chris and Hesse, Christopher and Fischer, Claudia and Chan, Clive and Roberts, Dan and Kappler, Daniel and Levy, Daniel and Selsam, Daniel and Dohan, David and Farhi, David and Mely, David and Robinson, David and Tsipras, Dimitris and Li, Doug and Oprica, Dragos and Freeman, Eben and Zhang, Eddie and Wong, Edmund and Proehl, Elizabeth and Cheung, Enoch and Mitchell, Eric and Wallace, Eric and Ritter, Erik and Mays, Evan and Wang, Fan and Such, Felipe Petroski and Raso, Filippo and Leoni, Florencia and Tsimpourlas, Foivos and Song, Francis and von Lohmann, Fred and Sulit, Freddie and Salmon, Geoff and Parascandolo, Giambattista and Chabot, Gildas and Zhao, Grace and Brockman, Greg and Leclerc, Guillaume and Salman, Hadi and Bao, Haiming and Sheng, Hao and Andrin, Hart and Bagherinezhad, Hessam and Ren, Hongyu and Lightman, Hunter and Chung, Hyung Won and Kivlichan, Ian and O'Connell, Ian and Osband, Ian and Gilaberte, Ignasi Clavera and Akkaya, Ilge and Kostrikov, Ilya and Sutskever, Ilya and Kofman, Irina and Pachocki, Jakub and Lennon, James and Wei, Jason and Harb, Jean and Twore, Jerry and Feng, Jiacheng and Yu, Jiahui and Weng, Jiayi and Tang, Jie and Yu, Jieqi and Candela, Joaquin Qui{\~n}onero and Palermo, Joe and Parish, Joel and Heidecke, Johannes and Hallman, John and Rizzo, John and Gordon, Jonathan and Uesato, Jonathan and Ward, Jonathan and Huizinga, Joost and Wang, Julie and Chen, Kai and Xiao, Kai and Singhal, Karan and Nguyen, Karina and Cobbe, Karl and Shi, Katy and Wood, Kayla and Rimbach, Kendra and {Gu-Lemberg}, Keren and Liu, Kevin and Lu, Kevin and Stone, Kevin and Yu, Kevin and Ahmad, Lama and Yang, Lauren and Liu, Leo and Maksin, Leon and Ho, Leyton and Fedus, Liam and Weng, Lilian and Li, Linden and McCallum, Lindsay and Held, Lindsey and Kuhn, Lorenz and Kondraciuk, Lukas and Kaiser, Lukasz and Metz, Luke and Boyd, Madelaine and Trebacz, Maja and Joglekar, Manas and Chen, Mark and Tintor, Marko and Meyer, Mason and Jones, Matt and Kaufer, Matt and Schwarzer, Max and Shah, Meghan and Yatbaz, Mehmet and Guan, Melody Y. and Xu, Mengyuan and Yan, Mengyuan and Glaese, Mia and Chen, Mianna and Lampe, Michael and Malek, Michael and Wang, Michele and Fradin, Michelle and McClay, Mike and Pavlov, Mikhail and Wang, Miles and Wang, Mingxuan and Murati, Mira and Bavarian, Mo and Rohaninejad, Mostafa and McAleese, Nat and Chowdhury, Neil and Chowdhury, Neil and Ryder, Nick and Tezak, Nikolas and Brown, Noam and Nachum, Ofir and Boiko, Oleg and Murk, Oleg and Watkins, Olivia and Chao, Patrick and Ashbourne, Paul and Izmailov, Pavel and Zhokhov, Peter and Dias, Rachel and Arora, Rahul and Lin, Randall and Lopes, Rapha Gontijo and Gaon, Raz and Miyara, Reah and Leike, Reimar and Hwang, Renny and Garg, Rhythm and Brown, Robin and James, Roshan and Shu, Rui and Cheu, Ryan and Greene, Ryan and Jain, Saachi and Altman, Sam and Toizer, Sam and Toyer, Sam and Miserendino, Samuel and Agarwal, Sandhini and Hernandez, Santiago and Baker, Sasha and McKinney, Scott and Yan, Scottie and Zhao, Shengjia and Hu, Shengli and Santurkar, Shibani and Chaudhuri, Shraman Ray and Zhang, Shuyuan and Fu, Siyuan and Papay, Spencer and Lin, Steph and Balaji, Suchir and Sanjeev, Suvansh and Sidor, Szymon and Broda, Tal and Clark, Aidan and Wang, Tao and Gordon, Taylor and Sanders, Ted and Patwardhan, Tejal and Sottiaux, Thibault and Degry, Thomas and Dimson, Thomas and Zheng, Tianhao and Garipov, Timur and Stasi, Tom and Bansal, Trapit and Creech, Trevor and Peterson, Troy and Eloundou, Tyna and Qi, Valerie and Kosaraju, Vineet and Monaco, Vinnie and Pong, Vitchyr and Fomenko, Vlad and Zheng, Weiyi and Zhou, Wenda and McCabe, Wes and Zaremba, Wojciech and Dubois, Yann and Lu, Yinghai and Chen, Yining and Cha, Young and Bai, Yu and He, Yuchen and Zhang, Yuchen and Wang, Yunyun and Shao, Zheng and Li, Zhuohan},
  year = 2024,
  number = {arXiv:2412.16720},
  eprint = {2412.16720},
  publisher = {arXiv},
  archiveprefix = {arXiv}
}

@misc{pan_TinyZero_2025,
  title = {{{TinyZero}}},
  author = {Pan, Jiayi and Zhang, Junjie and Wang, Xiangyao and Yuan, Lifan and Peng, Hao and Suhr, Alane},
  year = 2025
}

@misc{rafailov_DirectPreferenceOptimization_2023,
  title = {Direct {{Preference Optimization}}: {{Your Language Model}} Is {{Secretly}} a {{Reward Model}}},
  author = {Rafailov, Rafael and Sharma, Archit and Mitchell, Eric and Ermon, Stefano and Manning, Christopher D. and Finn, Chelsea},
  year = 2023,
  number = {arXiv:2305.18290},
  eprint = {2305.18290},
  publisher = {arXiv},
  archiveprefix = {arXiv}
}

@misc{roux_TaperedOffPolicyREINFORCE_2025,
  title = {Tapered {{Off-Policy REINFORCE}}: {{Stable}} and Efficient Reinforcement Learning for {{LLMs}}},
  author = {Roux, Nicolas Le and Bellemare, Marc G. and Lebensold, Jonathan and Bergeron, Arnaud and Greaves, Joshua and Fr{\'e}chette, Alex and Pelletier, Carolyne and {Thibodeau-Laufer}, Eric and Toth, S{\'a}ndor and Work, Sam},
  year = 2025,
  number = {arXiv:2503.14286},
  eprint = {2503.14286},
  publisher = {arXiv},
  archiveprefix = {arXiv}
}

@misc{schulman_LoRARegret_2025,
  title = {{{LoRA Without Regret}}},
  author = {Schulman, John and {al}, et},
  year = 2025,
  journal = {Thinking Machines Lab Blog Post},
  howpublished = {https://thinkingmachines.ai/blog/lora/},
  langid = {english}
}

@misc{schulman_ProximalPolicyOptimization_2017,
  title = {Proximal {{Policy Optimization Algorithms}}},
  author = {Schulman, John and Wolski, Filip and Dhariwal, Prafulla and Radford, Alec and Klimov, Oleg},
  year = 2017,
  number = {arXiv:1707.06347},
  eprint = {1707.06347},
  publisher = {arXiv},
  archiveprefix = {arXiv}
}

@misc{shao_DeepSeekMathPushingLimits_2024,
  title = {{{DeepSeekMath}}: {{Pushing}} the {{Limits}} of {{Mathematical Reasoning}} in {{Open Language Models}}},
  author = {Shao, Zhihong and Wang, Peiyi and Zhu, Qihao and Xu, Runxin and Song, Junxiao and Bi, Xiao and Zhang, Haowei and Zhang, Mingchuan and Li, Y. K. and Wu, Y. and Guo, Daya},
  year = 2024,
  number = {arXiv:2402.03300},
  eprint = {2402.03300},
  publisher = {arXiv},
  archiveprefix = {arXiv},
  langid = {english}
}

@misc{shao_SpuriousRewardsRethinking_2025,
  title = {Spurious {{Rewards}}: {{Rethinking Training Signals}} in {{RLVR}}},
  author = {Shao, Rulin and Li, Shuyue Stella and Xin, Rui and Geng, Scott and Wang, Yiping and Oh, Sewoong and Du, Simon Shaolei and Lambert, Nathan and Min, Sewon and Krishna, Ranjay and Tsvetkov, Yulia and Hajishirzi, Hannaneh and Koh, Pang Wei and Zettlemoyer, Luke},
  year = 2025,
  number = {arXiv:2506.10947},
  eprint = {2506.10947},
  publisher = {arXiv},
  archiveprefix = {arXiv}
}

@misc{team_KimiK15Scaling_2025,
  title = {Kimi K1.5: {{Scaling Reinforcement Learning}} with {{LLMs}}},
  author = {Team, Kimi and Du, Angang and Gao, Bofei and Xing, Bowei and Jiang, Changjiu and Chen, Cheng and Li, Cheng and Xiao, Chenjun and Du, Chenzhuang and Liao, Chonghua and Tang, Chuning and Wang, Congcong and Zhang, Dehao and Yuan, Enming and Lu, Enzhe and Tang, Fengxiang and Sung, Flood and Wei, Guangda and Lai, Guokun and Guo, Haiqing and Zhu, Han and Ding, Hao and Hu, Hao and Yang, Hao and Zhang, Hao and Yao, Haotian and Zhao, Haotian and Lu, Haoyu and Li, Haoze and Yu, Haozhen and Gao, Hongcheng and Zheng, Huabin and Yuan, Huan and Chen, Jia and Guo, Jianhang and Su, Jianlin and Wang, Jianzhou and Zhao, Jie and Zhang, Jin and Liu, Jingyuan and Yan, Junjie and Wu, Junyan and Shi, Lidong and Ye, Ling and Yu, Longhui and Dong, Mengnan and Zhang, Neo and Ma, Ningchen and Pan, Qiwei and Gong, Qucheng and Liu, Shaowei and Ma, Shengling and Wei, Shupeng and Cao, Sihan and Huang, Siying and Jiang, Tao and Gao, Weihao and Xiong, Weimin and He, Weiran and Huang, Weixiao and Xu, Weixin and Wu, Wenhao and He, Wenyang and Wei, Xianghui and Jia, Xianqing and Wu, Xingzhe and Xu, Xinran and Zu, Xinxing and Zhou, Xinyu and Pan, Xuehai and Charles, Y. and Li, Yang and Hu, Yangyang and Liu, Yangyang and Chen, Yanru and Wang, Yejie and Liu, Yibo and Qin, Yidao and Liu, Yifeng and Yang, Ying and Bao, Yiping and Du, Yulun and Wu, Yuxin and Wang, Yuzhi and Zhou, Zaida and Wang, Zhaoji and Li, Zhaowei and Zhu, Zhen and Zhang, Zheng and Wang, Zhexu and Yang, Zhilin and Huang, Zhiqi and Huang, Zihao and Xu, Ziyao and Yang, Zonghan and Lin, Zongyu},
  year = 2025,
  number = {arXiv:2501.12599},
  eprint = {2501.12599},
  publisher = {arXiv},
  archiveprefix = {arXiv}
}

@misc{team_KimiK2Open_2025,
  title = {Kimi {{K2}}: {{Open Agentic Intelligence}}},
  author = {Team, Kimi and Bai, Yifan and Bao, Yiping and Chen, Guanduo and Chen, Jiahao and Chen, Ningxin and Chen, Ruijue and Chen, Yanru and Chen, Yuankun and Chen, Yutian and Chen, Zhuofu and Cui, Jialei and Ding, Hao and Dong, Mengnan and Du, Angang and Du, Chenzhuang and Du, Dikang and Du, Yulun and Fan, Yu and Feng, Yichen and Fu, Kelin and Gao, Bofei and Gao, Hongcheng and Gao, Peizhong and Gao, Tong and Gu, Xinran and Guan, Longyu and Guo, Haiqing and Guo, Jianhang and Hu, Hao and Hao, Xiaoru and He, Tianhong and He, Weiran and He, Wenyang and Hong, Chao and Hu, Yangyang and Hu, Zhenxing and Huang, Weixiao and Huang, Zhiqi and Huang, Zihao and Jiang, Tao and Jiang, Zhejun and Jin, Xinyi and Kang, Yongsheng and Lai, Guokun and Li, Cheng and Li, Fang and Li, Haoyang and Li, Ming and Li, Wentao and Li, Yanhao and Li, Yiwei and Li, Zhaowei and Li, Zheming and Lin, Hongzhan and Lin, Xiaohan and Lin, Zongyu and Liu, Chengyin and Liu, Chenyu and Liu, Hongzhang and Liu, Jingyuan and Liu, Junqi and Liu, Liang and Liu, Shaowei and Liu, T. Y. and Liu, Tianwei and Liu, Weizhou and Liu, Yangyang and Liu, Yibo and Liu, Yiping and Liu, Yue and Liu, Zhengying and Lu, Enzhe and Lu, Lijun and Ma, Shengling and Ma, Xinyu and Ma, Yingwei and Mao, Shaoguang and Mei, Jie and Men, Xin and Miao, Yibo and Pan, Siyuan and Peng, Yebo and Qin, Ruoyu and Qu, Bowen and Shang, Zeyu and Shi, Lidong and Shi, Shengyuan and Song, Feifan and Su, Jianlin and Su, Zhengyuan and Sun, Xinjie and Sung, Flood and Tang, Heyi and Tao, Jiawen and Teng, Qifeng and Wang, Chensi and Wang, Dinglu and Wang, Feng and Wang, Haiming and Wang, Jianzhou and Wang, Jiaxing and Wang, Jinhong and Wang, Shengjie and Wang, Shuyi and Wang, Yao and Wang, Yejie and Wang, Yiqin and Wang, Yuxin and Wang, Yuzhi and Wang, Zhaoji and Wang, Zhengtao and Wang, Zhexu and Wei, Chu and Wei, Qianqian and Wu, Wenhao and Wu, Xingzhe and Wu, Yuxin and Xiao, Chenjun and Xie, Xiaotong and Xiong, Weimin and Xu, Boyu and Xu, Jing and Xu, Jinjing and Xu, L. H. and Xu, Lin and Xu, Suting and Xu, Weixin and Xu, Xinran and Xu, Yangchuan and Xu, Ziyao and Yan, Junjie and Yan, Yuzi and Yang, Xiaofei and Yang, Ying and Yang, Zhen and Yang, Zhilin and Yang, Zonghan and Yao, Haotian and Yao, Xingcheng and Ye, Wenjie and Ye, Zhuorui and Yin, Bohong and Yu, Longhui and Yuan, Enming and Yuan, Hongbang and Yuan, Mengjie and Zhan, Haobing and Zhang, Dehao and Zhang, Hao and Zhang, Wanlu and Zhang, Xiaobin and Zhang, Yangkun and Zhang, Yizhi and Zhang, Yongting and Zhang, Yu and Zhang, Yutao and Zhang, Yutong and Zhang, Zheng and Zhao, Haotian and Zhao, Yikai and Zheng, Huabin and Zheng, Shaojie and Zhou, Jianren and Zhou, Xinyu and Zhou, Zaida and Zhu, Zhen and Zhuang, Weiyu and Zu, Xinxing},
  year = 2025,
  number = {arXiv:2507.20534},
  eprint = {2507.20534},
  publisher = {arXiv},
  archiveprefix = {arXiv}
}

@misc{thrampoulidis_AdvantageShapingSurrogate_2025,
  title = {Advantage {{Shaping}} as {{Surrogate Reward Maximization}}: {{Unifying Pass}}@{{K Policy Gradients}}},
  author = {Thrampoulidis, Christos and Mahdavi, Sadegh and Deng, Wenlong},
  year = 2025,
  number = {arXiv:2510.23049},
  eprint = {2510.23049},
  publisher = {arXiv},
  archiveprefix = {arXiv}
}

@misc{vojnovic_WhatAlignmentObjective_2025,
  title = {What Is the {{Alignment Objective}} of {{GRPO}}?},
  author = {Vojnovic, Milan and Yun, Se-Young},
  year = 2025,
  number = {arXiv:2502.18548},
  eprint = {2502.18548},
  publisher = {arXiv},
  archiveprefix = {arXiv}
}

@misc{wang_ASPOAsymmetricImportance_2025,
  title = {{{ASPO}}: {{Asymmetric Importance Sampling Policy Optimization}}},
  author = {Wang, Jiakang and Liu, Runze and Lin, Lei and Hu, Wenping and Li, Xiu and Zhang, Fuzheng and Zhou, Guorui and Gai, Kun},
  year = 2025,
  number = {arXiv:2510.06062},
  eprint = {2510.06062},
  publisher = {arXiv},
  archiveprefix = {arXiv}
}

@article{williams_SimpleStatisticalGradientfollowing_1992,
  title = {Simple Statistical Gradient-Following Algorithms for Connectionist Reinforcement Learning},
  author = {Williams, Ronald J.},
  year = 1992,
  journal = {Machine Learning},
  volume = {8},
  number = {3},
  pages = {229--256},
  langid = {english}
}

@misc{yu_DAPOOpenSourceLLM_2025,
  title = {{{DAPO}}: {{An Open-Source LLM Reinforcement Learning System}} at {{Scale}}},
  author = {Yu, Qiying and Zhang, Zheng and Zhu, Ruofei and Yuan, Yufeng and Zuo, Xiaochen and Yue, Yu and Dai, Weinan and Fan, Tiantian and Liu, Gaohong and Liu, Lingjun and Liu, Xin and Lin, Haibin and Lin, Zhiqi and Ma, Bole and Sheng, Guangming and Tong, Yuxuan and Zhang, Chi and Zhang, Mofan and Zhang, Wang and Zhu, Hang and Zhu, Jinhua and Chen, Jiaze and Chen, Jiangjie and Wang, Chengyi and Yu, Hongli and Song, Yuxuan and Wei, Xiangpeng and Zhou, Hao and Liu, Jingjing and Ma, Wei-Ying and Zhang, Ya-Qin and Yan, Lin and Qiao, Mu and Wu, Yonghui and Wang, Mingxuan},
  year = 2025,
  number = {arXiv:2503.14476},
  eprint = {2503.14476},
  publisher = {arXiv},
  archiveprefix = {arXiv}
}

@inproceedings{zhou_CoDaPOConfidenceDifficultyAdaptive_2025,
  title = {{{CoDaPO}}: {{Confidence}} and {{Difficulty-Adaptive Policy Optimization}} for {{Post-Training Language Models}}},
  booktitle = {2nd {{AI}} for {{Math Workshop}} @ {{ICML}} 2025},
  author = {Zhou, Zhanke and Lu, Xiangyu and Cao, Chentao and Miranda, Brando and Liu, Tongliang and Han, Bo and Koyejo, Sanmi},
  year = 2025,
  langid = {english}
}

\newpage

\appendix

\section{Supplementary figures and table}

Figure~\ref{fig:rhohatdistribution_math} and \ref{fig:rhohatdistributionDAPO} show the distribution of the fraction of correct responses during Linear-R runs on the MATH and DAPO-math experiments discussed in the main body. 

Table~\ref{tab:effective_weights_final} contains the effective weights for the considered algorithms.

\begin{figure*}
\begin{center}
\begin{tikzpicture}

\pgfplotstableread{
bin_start step_5 step_10 step_20 step_45 step_90 step_180
0.0000 13.67 10.16 13.67 10.16 10.55 12.11
0.0625 6.64 2.73 2.73 4.69 2.73 1.56
0.1250 3.52 4.69 3.52 1.56 2.34 0.78
0.1875 5.08 3.52 3.12 2.73 0.78 0.78
0.2500 3.52 2.73 1.95 2.73 2.34 2.34
0.3125 4.69 2.73 1.17 1.56 1.95 1.95
0.3750 3.91 4.30 2.34 1.95 1.56 0.39
0.4375 3.91 3.52 2.73 1.17 0.78 3.91
0.5000 4.30 4.69 2.73 5.08 3.12 2.34
0.5625 2.34 4.30 3.12 4.69 1.56 3.12
0.6250 7.42 5.47 3.12 1.95 1.56 2.73
0.6875 4.30 3.12 5.08 2.73 4.69 1.95
0.7500 4.69 4.69 2.73 3.52 2.73 1.56
0.8125 5.86 8.59 5.86 3.91 3.52 5.08
0.8750 7.42 7.03 10.94 4.30 8.59 6.25
0.9375 7.81 9.38 13.28 9.77 12.11 8.20
1.0000 10.94 18.36 21.88 37.50 39.06 44.92
}\datatable

\begin{groupplot}[
group style={
        group size=6 by 1,
        horizontal sep=0.7cm,
        ylabels at=edge left,
        yticklabels at=edge left,
        horizontal sep=0.1cm,
    },
    width=4cm,
    height=4cm,
    ymin=0,
    ymax=80,
    ylabel near ticks,
    xlabel near ticks,
    xtick={0,0.25,0.5,0.75,1},
    xticklabels={$0$,$\frac{4}{16}$,$\frac{8}{16}$,$\frac{12}{16}$,$1$},
    xmin=-0.05,
    xmax=1.05,
    tick label style={font=\small},
    label style={font=\small},
    title style={font=\small\bfseries},
]

\nextgroupplot[title={Step 5}, ylabel={Percentage (\%)}]
\addplot[thick, blue, mark=*, mark size=1pt, fill=blue, fill opacity=0.3] table[x=bin_start, y=step_5] {\datatable} \closedcycle;

\nextgroupplot[title={Step 10}]
\addplot[thick, blue, mark=*, mark size=1pt, fill=blue, fill opacity=0.3] table[x=bin_start, y=step_10] {\datatable} \closedcycle;

\nextgroupplot[title={Step 20}]
\addplot[thick, blue, mark=*, mark size=1pt, fill=blue, fill opacity=0.3] table[x=bin_start, y=step_20] {\datatable} \closedcycle;

\nextgroupplot[title={Step 45}]
\addplot[thick, blue, mark=*, mark size=1pt, fill=blue, fill opacity=0.3] table[x=bin_start, y=step_45] {\datatable} \closedcycle;

\nextgroupplot[title={Step 90}]
\addplot[thick, blue, mark=*, mark size=1pt, fill=blue, fill opacity=0.3] table[x=bin_start, y=step_90] {\datatable} \closedcycle;

\nextgroupplot[title={Step 180}]
\addplot[thick, blue, mark=*, mark size=1pt, fill=blue, fill opacity=0.3] table[x=bin_start, y=step_180] {\datatable} \closedcycle;
\end{groupplot}
\end{tikzpicture}
\end{center}
\vspace{-0.35cm}
\caption{
\label{fig:rhohatdistribution_math}
Distribution of the fraction of correct completions, $\rhohat$, out of the 16 completions for each prompt during a run of Linear-R for MATH 
}
\end{figure*}
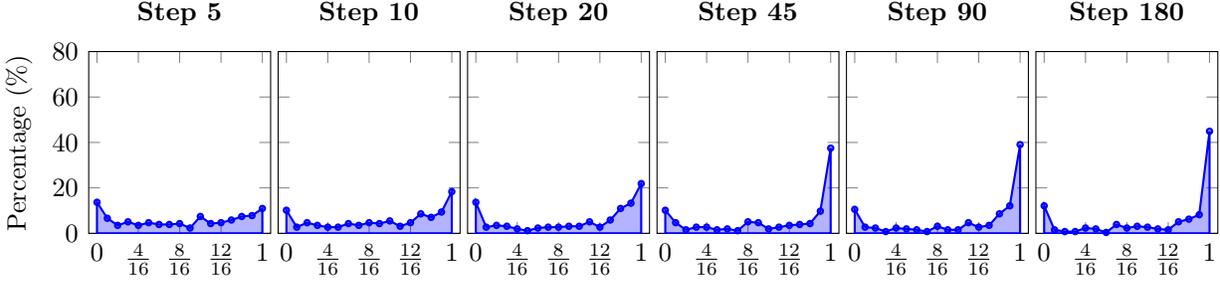

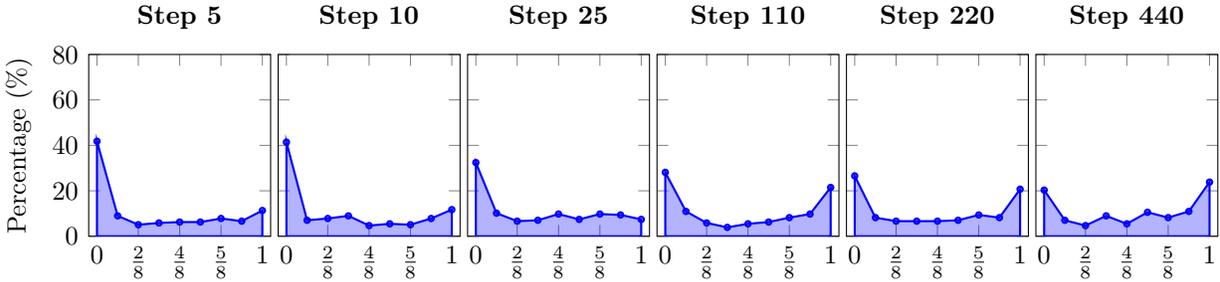
\begin{figure*}
\begin{center}
\begin{tikzpicture}

\pgfplotstableread{
bin_start step_5 step_10 step_25 step_110 step_220 step_440
0.0000 41.80 41.41 32.42 28.12 26.56 20.31
0.1250 8.98 7.03 10.16 10.94 8.20 7.03
0.2500 5.08 7.81 6.64 5.86 6.64 4.69
0.3750 5.86 8.98 7.03 3.91 6.64 8.98
0.5000 6.25 4.69 9.77 5.47 6.64 5.47
0.6250 6.25 5.47 7.42 6.25 7.03 10.55
0.7500 7.81 5.08 9.77 8.20 9.38 8.20
0.8750 6.64 7.81 9.38 9.77 8.20 10.94
1.0000 11.33 11.72 7.42 21.48 20.70 23.83
}\datatable

\begin{groupplot}[
    group style={
        group size=6 by 1,
        horizontal sep=0.7cm,
        ylabels at=edge left,
        yticklabels at=edge left,
        horizontal sep=0.1cm,
    },
    width=4cm,
    height=4cm,
    ymin=0,
    ymax=80,
    ylabel near ticks,
    xlabel near ticks,
    xtick={0,0.25,0.5,0.75,1},
    xticklabels={$0$,$\frac{2}{8}$,$\frac{4}{8}$,$\frac{5}{8}$,$1$},
    xmin=-0.05,
    xmax=1.05,
    tick label style={font=\small},
    label style={font=\small},
    title style={font=\small\bfseries},
]

\nextgroupplot[title={Step 5}, ylabel={Percentage (\%)}]
\addplot[thick, blue, mark=*, mark size=1pt, fill=blue, fill opacity=0.3] table[x=bin_start, y=step_5] {\datatable} \closedcycle;

\nextgroupplot[title={Step 10}]
\addplot[thick, blue, mark=*, mark size=1pt, fill=blue, fill opacity=0.3] table[x=bin_start, y=step_10] {\datatable} \closedcycle;

\nextgroupplot[title={Step 25}]
\addplot[thick, blue, mark=*, mark size=1pt, fill=blue, fill opacity=0.3] table[x=bin_start, y=step_25] {\datatable} \closedcycle;

\nextgroupplot[title={Step 110}]
\addplot[thick, blue, mark=*, mark size=1pt, fill=blue, fill opacity=0.3] table[x=bin_start, y=step_110] {\datatable} \closedcycle;

\nextgroupplot[title={Step 220}]
\addplot[thick, blue, mark=*, mark size=1pt, fill=blue, fill opacity=0.3] table[x=bin_start, y=step_220] {\datatable} \closedcycle;

\nextgroupplot[title={Step 440}]
\addplot[thick, blue, mark=*, mark size=1pt, fill=blue, fill opacity=0.3] table[x=bin_start, y=step_440] {\datatable} \closedcycle;
\end{groupplot}
\end{tikzpicture}
\end{center}
\vspace{-0.35cm}
\caption{
\label{fig:rhohatdistributionDAPO}
Distribution of the fraction of correct completions, $\rhohat$, out of the 8 completions for each prompt during a run of Linear-R for DAPO math (see Section~\ref{sec:DAPOmath}).
}
\end{figure*}

\begin{table*}[tb]
    \centering
    \caption{Overview of algorithms' \emph{effective} weights multiplying the difference of averaged positive and negative gradients $\widehat \nabla_1 - \widehat \nabla_0$ for binary rewards.}
    \label{tab:effective_weights_final}
    \renewcommand{\arraystretch}{1.5} 
    \setlength{\tabcolsep}{4pt}       
    
    \resizebox{\textwidth}{!}{%
    \begin{tabular}{lcccccc}
        \toprule
        & \multicolumn{4}{c}{\textbf{Proposed  Weightings}} & \multicolumn{2}{c}{\textbf{Standard Baselines}} \\
        \cmidrule(lr){2-5} \cmidrule(lr){6-7}
        \textbf{Algorithm} & \textbf{Linear-R} & \textbf{Sqrt-R} & \textbf{Plateau-R} & \textbf{Uniform-R} & \textbf{GRPO} & \textbf{RLOO} \\
        \midrule
        \begin{tabular}{@{}l@{}}\textbf{Effective} \\ \textbf{Weight}\end{tabular} & 
        $1 - \hat{\rho}$ & 
        $\sqrt{1 - \hat{\rho}}$ & 
        $\begin{cases} 1/2, & \hat{\rho} < 0.5 \\ \sqrt{\hat{\rho}(1-\hat{\rho})}, & \hat{\rho} \ge 0.5 \end{cases}$ & 
        $1$ & 
        $\sqrt{\hat{\rho}(1-\hat{\rho})}$ & 
        $\hat{\rho}(1-\hat{\rho})$ \\
        \bottomrule
    \end{tabular}%
    \vspace{-0.1in}
    }
\end{table*}

\section{Ablation studies}
In this section, we present ablation studies.

\subsection{Importance of assigning non-zero weights to zero-success groups}

The asymmetric weightings we consider in this paper assign non-zero weight to prompts with $\rhohat=0$. As mentioned this is important for performance. To demonstrate this, we performed further TinyZero experiments for Plateau-R, for which we assign zero weight to zero-success prompts (i.e., groups with $\rhohat=0$). 

See Figure~\ref{fig:tinyzerocomparison_ablation} for the results. It can be seen that with the same stepsize, training is less stable, and leads to worse performance. A smaller stepsize fixes the instability, but results in worse performance. 

We also considered the weight $\omega_x=\sqrt{(1-\rho)/\rho}$, which is another example of an asymmetric weight that assigns zero weight to zero-success groups. We call this kimi-weighting since the prioritized sampling strategy used by \citet{team_KimiK15Scaling_2025} where samples are reweighted by $1-\hat \rho$, can be interpreted through surrogate reward maximization~\cite{thrampoulidis_AdvantageShapingSurrogate_2025} (cf. Section~\ref{sec:surrogate reward perspective}) yielding that weight. 

As seen in Figure~\ref{fig:tinyzerocomparison_ablation}, kimi-weighting performs worse than the other asymmetric weightings.

\begin{figure}[tb]
    \centering

\usepgfplotslibrary{statistics}

\pgfplotsset{
    smoothcurve/.style={
        thick,
        smooth
    }
}

\begin{tikzpicture}
    \begin{groupplot}[
        group style={
            group size=2 by 1,
            horizontal sep=0.9cm,
            xlabels at=edge bottom,
            ylabels at=edge left,
        },
        width=7.5cm,
        height=5.5cm,
        xlabel={Step},
        ylabel={Pass@1},
        grid=major,
        grid style={dashed, gray!30},
        legend style={
            at={(1.9,0.1)},
            anchor=south east,
            font=\small,
        },
        table/col sep=comma,
    ]
    
    \nextgroupplot[xmin=0, xmax=100, title={Early training},width=4.5cm]
        \addplot[red, thick] table[x=Step, y=a_reinforce] {./fig/tinyzero_comparison.csv};
        \addplot[orange, thick] table[x=Step, y=a_advantage_0_1e-6] {./fig/tinyzero_comparison.csv};
        \addplot[yellow, thick] table[x=Step, y=a_advantage_0_5e-7] {./fig/tinyzero_comparison.csv};
        \addplot[green, thick] table[x=Step, y=a_advantage_0_75e-8] {./fig/tinyzero_comparison.csv};
        \addplot[blue, thick] table[x=Step, y=grpo] {./fig/tinyzero_comparison.csv};
        \addplot[purple, thick] table[x=Step, y=kimi_reinforce] {./fig/tinyzero_comparison.csv};
    
    \nextgroupplot[xmin=100, title={Later training}, ylabel={}]
        \addplot[smoothcurve,red, thick] table[x=Step, y=a_reinforce] {./fig/tinyzero_comparison.csv};
        \addlegendentry{Plateau-R,  lr 1e-6}
        \addplot[smoothcurve,orange, thick] table[x=Step, y=a_advantage_0_1e-6] {./fig/tinyzero_comparison.csv};
        \addlegendentry{Plateau-R, Zero,  lr 1e-6}
        \addplot[smoothcurve,yellow, thick] table[x=Step, y=a_advantage_0_5e-7] {./fig/tinyzero_comparison.csv};
        \addlegendentry{Plateau-R, Zero,  lr 0.5e-6}
        \addplot[smoothcurve,green, thick] table[x=Step, y=a_advantage_0_75e-8] {./fig/tinyzero_comparison.csv};
        \addlegendentry{Plateau-R, Zero,  lr 0.75e-6}
        \addplot[smoothcurve,blue, thick] table[x=Step, y=grpo] {./fig/tinyzero_comparison.csv};
        \addlegendentry{GRPO}
        \addplot[smoothcurve,purple, thick] table[x=Step, y=kimi_reinforce] {./fig/tinyzero_comparison.csv};
        \addlegendentry{kimi-REINFORCE}
    
    \end{groupplot}
\end{tikzpicture}
    \caption{From-scratch RL: TinyZero, ablation study on the importance of assigning non-zero weights to zero-success groups. Plateau-R modified with assigning zero weight to zero-success groups ($\hat \rho=0$) performs worse and is less stable; adjusting the learning rate does not fix this. Still, this modification is better than GRPO, so the 
    }
    \label{fig:tinyzerocomparison_ablation}
    \vspace{-0.4cm}
\end{figure}
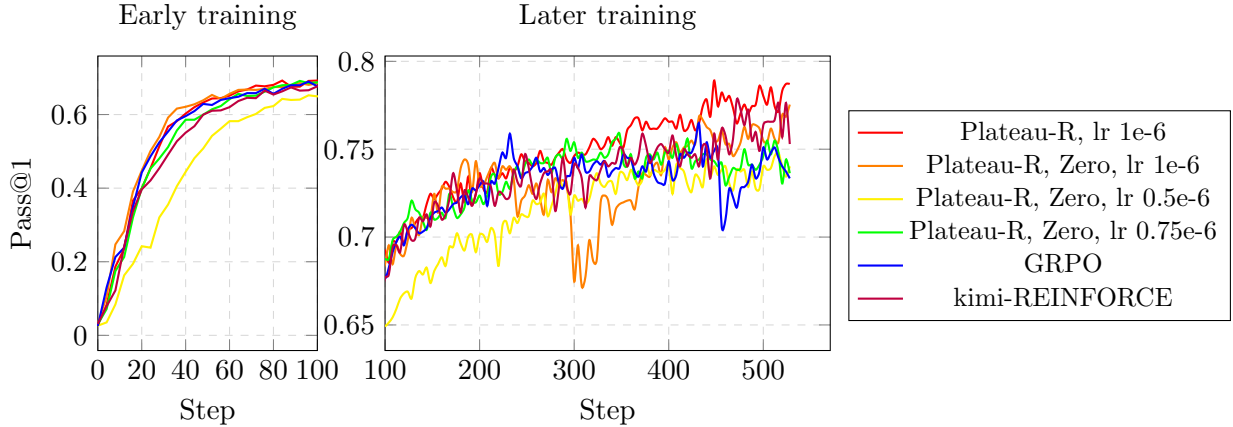

\subsection{Run-to-run differences}

There are of course run-to-run-differences for our results; to demonstrate that those are relatively modest, we performed three independent runs of Linear-R for TinyZero. The results in Figure~\ref{fig:tinyzerocomparison_run_variations} demonstrate that the run-to-run differences are small and thus the gains we observe over algorithms with symmetric prompt weightings are consistent.

\begin{figure}[tb]
    \centering

\begin{tikzpicture}
    \begin{groupplot}[
        group style={
            group size=2 by 1,
            horizontal sep=0.7cm,
            xlabels at=edge bottom,
            ylabels at=edge left,
        },
        width=7.5cm,
        height=5.5cm,
        xlabel={Step},
        ylabel={Pass@1},
        grid=major,
        grid style={dashed, gray!30},
        legend style={
            at={(1.5,0.1)},
            anchor=south east,
            font=\small,
        },
        table/col sep=comma,
    ]
    
    \nextgroupplot[xmin=0, xmax=50, title={Early training},width=3.5cm]
        \addplot[red, thick] table[x=Step, y=log_reinforce1] {./fig/tinyzero_comparison.csv};
        \addplot[red, thick] table[x=Step, y=log_reinforce2] {./fig/tinyzero_comparison.csv};
        \addplot[red, thick] table[x=Step, y=log_reinforce3] {./fig/tinyzero_comparison.csv};
    
    \nextgroupplot[xmin=50, title={Later training}, ylabel={}]
        \addplot[red, thick] table[x=Step, y=log_reinforce1] {./fig/tinyzero_comparison.csv};
        \addlegendentry{Linear-R R1}
        \addplot[red, thick] table[x=Step, y=log_reinforce2] {./fig/tinyzero_comparison.csv};
        \addlegendentry{Linear-R R2}
        \addplot[red, thick] table[x=Step, y=log_reinforce3] {./fig/tinyzero_comparison.csv};
        \addlegendentry{Linear-R R3}
    
    \end{groupplot}
\end{tikzpicture}
    \caption{From-scratch RL: TinyZero. Test error (i.e., fraction of correctly solved problems, Pass@1) during reinforcement learning. Three independent runs of Linear-R are shown to demonstrate that the run-to-run differences are small and gains are consistent.}
    \label{fig:tinyzerocomparison_run_variations}
    \vspace{-0.4cm}
\end{figure}
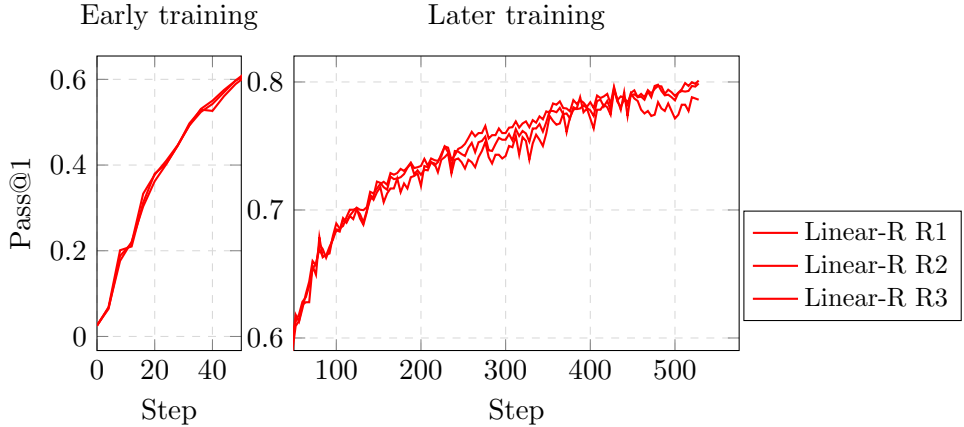

\subsection{Importance of gradient direction}
\label{sec:graddir}
In this work we vary the gradient weights $\omega_x(\rho)$ and choose the gradient direction $\hat d_x$ as the direction used by RLOO and GRPO. 
Choosing the gradient direction as the REINFORCE direction (see Section~\ref{sec:policyopt}, which effectively does not leverage gradients corresponding to zero-rewards) yields significantly worse performance: See Figure~\ref{fig:tinyzerocomparison_run_dir} for a comparison of  Linear-R vs the Linear-R weight with the REINFORCE direction (rejection sampling). 
We refer to this as rejection sampling as it corresponds to the log-type surrogate advocated in~\cite{davis_WhatObjectiveReasoning_2025} and their proposed implementation via rejection sampling. See also Section \ref{sec:onLR}.

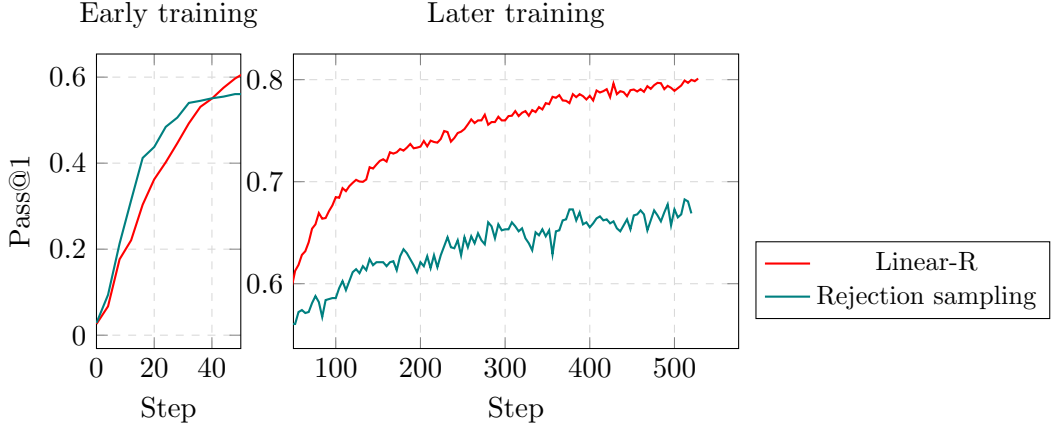
\begin{figure}[tb]
    \centering

\begin{tikzpicture}
    \begin{groupplot}[
        group style={
            group size=2 by 1,
            horizontal sep=0.7cm,
            xlabels at=edge bottom,
            ylabels at=edge left,
        },
        width=7.5cm,
        height=5.5cm,
        xlabel={Step},
        ylabel={Pass@1},
        grid=major,
        grid style={dashed, gray!30},
        legend style={
            at={(1.7,0.1)},
            anchor=south east,
            font=\small,
        },
        table/col sep=comma,
    ]
    
    \nextgroupplot[xmin=0, xmax=50, title={Early training},width=3.5cm]
        \addplot[red, thick] table[x=Step, y=log_reinforce1] {./fig/tinyzero_comparison.csv};
        \addplot[teal, thick] table[x=Step, y=rejection_sampling] {./fig/tinyzero_comparison.csv};
    
    \nextgroupplot[xmin=50, title={Later training}, ylabel={}]
        \addplot[red, thick] table[x=Step, y=log_reinforce1] {./fig/tinyzero_comparison.csv};
        \addlegendentry{Linear-R}
        \addplot[teal, thick] table[x=Step, y=rejection_sampling] {./fig/tinyzero_comparison.csv};
        \addlegendentry{Rejection sampling}
    
    \end{groupplot}
\end{tikzpicture}
    \caption{From-scratch RL: TinyZero. Test error (i.e., fraction of correctly solved problems, Pass@1) during reinforcement learning. 
    Linear-R vs Rejection sampling, }
    \label{fig:tinyzerocomparison_run_dir}
\end{figure}

\section{Experimental details}
\label{app:experimental_details}
Here we briefly describe the experimental details. As mentioned, for our RL experiments, we used the SkyRL codebase~\cite{cao2025skyrl}, and we have implemented our advantage estimators withing that framework; the code is in the supplementary material. 

The hyperparameters used in the experiments are in Table~\ref{tab:rlexperiments}. We used the same hyperparameters (such as batchsize, etc) within each of the four experimental setups for different advantage estimators, appart from the stepsize, where we considered and used different ones. Specifically we considered (2e-6,1e-6,5e-7) for some experiments for the different advantage estimators. 

For the TinyZero experiment the stepsizes for the runs shown are 1e-6 for each advantage estimator but for RLOO, where it is 2e-6.
Looking at effective weights in Figure~\ref{fig:gradient_weights} it makes sense that a stepsize that is good for RLOO should be twice of that of GRPO, given that the effective weights are very similar if we multiply the RLOO effective weight with 2. 
Applying the same thinking to Plateau-R, it might seem reasonable to decrese the stepsize of Plateau-R too relative to the other advantage estimators but here the situation is different: the effective batchsize (corresponding to completions with non-zero weight) for the Plateau-R is larger than for RLOO, and for larger batch sizes the stepsize can often be larger too. 
 
We trained on 1-2 nodes each with 4xH100 and our experiments altogether took 12k H100 hours.

\begin{table}[ht]
\centering
\begin{tabular}{lllll}
\hline
Hyperparameter & TinyZero & GSM8K & MATH & DAPO Math \\
\hline
Use KL Loss & No & No & No & No \\
Train Batch Size & 256 & 128/256 & 128 & 256 \\
Policy Mini Batch Size & 64 & 32/64 & 32 & 64 \\
Micro Train Batch Size / GPU & 8 & 4 & 8 & 4 \\
Micro Forward Batch Size / GPU & 8 & 4 & 8 & 4 \\
Max Prompt Length & 512 & 1024 & 1024 & 1024 \\
Max Response Length & 1024 & 2048 & 3072 & 7168 \\
Update Epochs Per Batch & 1 & 1 & 1 & 1 \\
Learning Rate & 1e-6 & 1e-6 & 5e-7 & 1e-6 \\
Max Grad Norm (Clip) & 0.5 & 1.0 & 1.0 & 1.0 \\
Clip Ratio Range & [0.2, 0.2] & [0.0, 0.0] & [0.2, 0.2] & [0.0, 0.0] \\
Sampling Temperature & 1.0  & 1.0  & 1.0  & 1.0  \\
Responses per prompt $\M$ & 16 & 16 & 16 & 8 \\
\hline
\end{tabular}
\caption{Hyperparameters for the fours sets of RL experiment.
\label{tab:rlexperiments}
}
\end{table}

\section{Proofs}

\subsection{The continuous limit of success-rate dynamics: Proof of Equation \eqref{eq:ODE}}\label{sec:get rho_t ODE}
The optimization problem in \eqref{eq:start_opt} is concave. Thus, rewriting its first-order optimality condition gives the following recursion for the optimal policies (see  \citep{mroueh_ReinforcementLearningVerifiable_2025}):
\begin{align}\label{eq:policy recursion}
&\pi_t(y|x) = \frac{\pi_{t-1}(y|x)\cdot \exp\left(\frac{1}{\beta}\cdot A_y(\rho_{t-1})\right)}{\E_{y'\sim\pi_{t-1}(\cdot|x)}\pi_{t-1}(y'|x)\exp\left(\frac{1}{\beta}\cdot A_{y'}(\rho_{t-1})\right)}
\\
&\,\,=\frac{\pi_{t-1}(y|x)\cdot \exp\left(\frac{1}{\beta}\cdot A_y(\rho_{t-1})\right)}{\rho_{t-1}\exp\left(\frac{1}{\beta}\cdot A_1(\rho_{t-1})\right)+(1-\rho_{t-1})\exp\left(\frac{1}{\beta}\cdot A_0(\rho_{t-1})\right)}\,.\nn
\end{align}
In the second equality, we used that rewards are binary, we recalled the definition 
$$\rho_{t-1}=\E_{y\sim\pi_{t-1}(\cdot|x)}\indicator{r(y)=1}\,,$$ and let $A_{0/1}$ be the advantage scores of wrong/correct responses. 
This expression makes explicit how the next policy re-weights correct and incorrect responses
relative to the previous policy. Importantly, because the update depends on the policy only through
its current success rate, the induced dynamics can be described entirely in terms of how this
success rate evolves over time. Concretely, using $\rho_t=\E_{y\sim\pi_t(\cdot|x)}\indicator{r(y)=1}=\int \mathrm{d}\pi_t(\cdot|x) \indicator{r(y)=1}$ yields the following recursion over success rates:
\begin{align}
    \rho_t  
    =\frac{1}{1 + \frac{1-\rho_{t-1}}{\rho_{t-1}} \exp\left(\frac{1}{\beta}\left(A_0(\rho_{t-1})-A_1(\rho_{t-1})\right)\right)}\nn\,.
\end{align}
Substituting the general advantage expressions from Eq. \eqref{eq:weighted emp advantage form} for weight function $\omega(\rho)$ we get the following recursion for the population (over model responses) success rate:
\begin{align}
\rho_t=\frac{1}{1 + \frac{1-\rho_{t-1}}{\rho_{t-1}} \exp\left(-\frac{1}{\beta}\omega(\rho_{t-1})\right)}\nn\,.
\end{align}

To further simplify the dynamics, it is convenient to work with the log-odds of the success
probability,
\[
L_t := \log\frac{\rho_t}{1-\rho_t}.
\]
This change of variables linearizes multiplicative updates in probability space and reveals that
learning proceeds additively in log-odds space.
$$\frac{1-\rho_t}{\rho_t} = \frac{1-\rho_{t-1}}{\rho_{t-1}} \exp\left(-\frac{1}{\beta}\omega(\rho_{t-1})\right)  \implies L_t = L_{t-1} + \frac{1}{\beta} \omega(\rho_{t-1})\,.$$
The log-odds of the success rate increases by a step size proportional to the gradient of the surrogate reward. The derivation up to this point follows \citep[Thm.~2]{mroueh_ReinforcementLearningVerifiable_2025}, extended here to general weight functions $\omega(\rho)$.

To gain further intuition about these dynamics and the role played by the weight, we consider a continuous-time limit of the
log-odds update. Concretely, let us consider a continuous-time limit of the form
\[
\dt{L_t} = \frac{1}{\beta} \omega(\rho_{t}) \;\implies\; \dt{\rho_t} = \frac{1}{\beta}\rho_t(1-\rho_t)\omega(\rho_t)\,,
\]
where for the latter implication we used the chain rule. 


\subsection{Dynamics for regular time for RLOO, Linear-R, Sqrt-R, and GRPO}

\begin{proposition}
\label{prop:ode_solutions_regular_time}
Solving the ODE~\eqref{eq:ODE}, i.e.,
\begin{align*}
\dt{\rho_t} = \frac{1}{\beta}\rho_t(1-\rho_t)\cdot \omega(\rho_t)\,.
\end{align*}
for various weights $\omega$ and for initialization $\rho_0\in(0,1)$ gives:
\begin{itemize}
    \item RLOO, $\omega(\rho)=1$:
    $\rho_t = \left({1 + \left(\frac{1-\rho_0}{\rho_0}\right) e^{-t/\beta}}\right)^{-1}$
    \item Linear-R, $\omega(\rho)=\frac{1}{\rho}$: 
    $\rho_t = 1 - (1-\rho_0)e^{-t/\beta}$
    \item Sqrt-R, $\omega(\rho)=\frac{1}{\rho\sqrt{1-\rho}}$: $\rho_t = 1-\left(\max\left\{\sqrt{1-\rho_0}-\frac{t}{2\beta},0\right\}\right)^2$
    \item  GRPO, $\omega(\rho)=\frac{1}{\sqrt{\rho(1-\rho)}}$:
    $\rho_t = \sin^2\left( \min\left( \frac{\pi}{2}, \frac{t}{2\beta} + \arcsin(\sqrt{\rho_0}) \right) \right)$
\end{itemize}
\end{proposition}
\begin{proof} We prove each case separately. 

\textbf{RLOO:}
For RLOO, the ODE becomes
\begin{align*}
\dt{\rho} = \frac{1}{\beta} \rho (1-\rho),
\end{align*}
and integration yields
\begin{align*}
\int \frac{1}{\rho (1-\rho)} d\rho = \int \frac{1}{\beta} dt.
\end{align*}
Solving this gives
\begin{align*}
\log \rho - \log(1-\rho) - (\log \rho_0 - \log(1-\rho_0)) = \frac{t}{\beta},
\end{align*}
which after some algebra gives the stated solution for RLOO above.

\textbf{Linear-R:} Next consider Linear-R, for which the ODE becomes
\begin{align*}
\dt{\rho} = \frac{1}{\beta} (1-\rho).
\end{align*}
This is a linear ODE 
\begin{align*}
\frac{d}{dt} (1-\rho) = - \frac{1}{\beta} (1-\rho),
\end{align*}
with solution
\begin{align*}
(1-\rho) = (1-\rho_0) e^{-t/\beta}\,.
\end{align*}
which after re-arranging yields the expression for Linear-R. 

\textbf{Sqrt-R:} For Sqrt-R, the ODE becomes
\begin{align*}
    \frac{d\rho}{dt}=\frac{1}{\beta}\sqrt{1-\rho}
\end{align*}
Let \(u_t:=1-\rho_t\). 
Then, while \(u_t>0\): \[ \frac{d u_t}{dt}=-\frac{1}{\beta}\sqrt{u_t}\,\,\Rightarrow\,\,
2\sqrt{u_t}=2\sqrt{u_0}-\frac{1}{\beta}t\,.\] Therefore
\({u_t}=\left(\max\{\sqrt{u_0}-\frac{t}{2\beta},0\}\right)^2\) which finishes the proof.

\textbf{GRPO:} For GRPO, the ODE becomes
\[
\frac{d\rho}{dt} 
= \frac{1}{\beta}\sqrt{\rho(1-\rho)}.
\]
Let $\theta_t=\arcsin(\sqrt{\rho_t})$. Then $\sin^2(\theta)=\rho$, and by the chain rule,
\[
\frac{d\rho}{dt}
= \frac{d}{dt}\big(\sin^2\theta\big)
= 2\sin\theta\cos\theta \,\frac{d\theta}{dt}
= 2\sqrt{\rho}\sqrt{1-\rho}\,\frac{d\theta}{dt}
\]
Equating this with the ODE gives
\[
2\sqrt{\rho(1-\rho)}\,\frac{d\theta}{dt}
\;=\; \frac{1}{\beta}\sqrt{\rho(1-\rho)}.
\]
Thus, for $\rho_t\in(0,1)$ we obtain the linear ODE
\[
\frac{d\theta}{dt} \;=\; \frac{1}{2\beta}.
\]
Integrating and using $\theta_{t=0}=\arcsin(\sqrt{\rho_0})$ yields
\[
\theta(t) \;=\; \frac{t}{2\beta} + \arcsin(\sqrt{\rho_0}).
\]
This derivation is valid as long as $\rho_t < 1$, which corresponds to $\theta_t < \pi/2$. Once $\theta_t$ reaches $\pi/2$, $\rho_t$ reaches $1$. Since $\rho=1$ is an equilibrium point, the solution stays at $1$ thereafter.
Therefore, the general solution is
\[
\rho_t \;=\; \sin^2\left( \min\left( \frac{\pi}{2}, \frac{t}{2\beta} + \arcsin(\sqrt{\rho_0}) \right) \right)\,,
\]
which concludes the proof for GRPO.
\end{proof}

\subsection{Effective time dynamics for Linear-R, Sqrt-R, RLOO, and GRPO}

\begin{proposition}
\label{prop:closed_form_tau}
Solving the effective-time ODE~\eqref{eq:effective_time_ode}, i.e., 
\begin{align*}
\frac{\mathrm{d}\rho_\tau}{\mathrm{d}\tau} = \frac{1}{\beta}\,\rho_\tau^2(1-\rho_\tau) \cdot \omega(\rho_\tau)
\end{align*}
for various weights $\omega$ and for initialization $\rho_{\tau=0}=\rho_0\in(0,1)$ gives:
\begin{itemize}
    \item Linear-R, $\omega(\rho)=\frac{1}{\rho}$:
    $
        \rho_\tau
        =
        \frac{1}{1+\left(\frac{1-\rho_0}{\rho_0}\right)e^{-\tau/\beta}}
    $

    \item Sqrt-R, $\omega(\rho)=\frac{1}{\rho\sqrt{1-\rho}}$: Define $\tau_*:=2\beta\operatorname{arctanh}\big(\sqrt{1-\rho_0}\big)$. Then,
    \begin{align*}
\rho_\tau
=
\begin{cases}
\operatorname{sech}^2\left(
\operatorname{arctanh}\big(\sqrt{1-\rho_0}\big)
-\frac{\tau}{2\beta}
\right),
& \tau \le \tau_\star, \\[1.2ex]
1,
& \tau > \tau_\star .
\end{cases}
\end{align*}

    \item GRPO, $\omega(\rho)=\frac{1}{\sqrt{\rho(1-\rho)}}$:
    Define $z_0:=\sqrt{\frac{1-\rho_0}{\rho_0}}$ and $\tau_*:=2\beta z_0$. Then, 
    \begin{align*}
        \rho_\tau
        =
        \frac{1}{1+\left(\max \left\{z_0-\frac{\tau}{2\beta},\,0\right\}\right)^2}
    \end{align*}
    In particular, $\rho_\tau=1$ for all $\tau\ge \tau_*$.

    \item RLOO, $\omega(\rho)=1$:
    Define
    $ 
    s_\tau
        :=
        \frac{\tau}{\beta}
        -\frac{1}{\rho_0}
        +\log \left(\frac{\rho_0}{1-\rho_0}\right).
    $
    Then $\rho_\tau$ is uniquely determined by the implicit relation
    \begin{align*}
        -\frac{1}{\rho_\tau}
        +\log\left(\frac{\rho_\tau}{1-\rho_\tau}\right)
        \;=\;
        s_\tau,
    \end{align*}
    and admits an explicit form in terms of the principal branch of the Lambert-$W$ function:
    \begin{align*}
        \rho_\tau
        =
        \frac{1}{1+W\left(\exp\left(-s_\tau-1\right)\right)}\,.
    \end{align*}
\end{itemize}
\end{proposition}
\begin{proof} We prove each case separately.

\textbf{Linear-R: $\omega(\rho)=1/\rho$.}
The ODE becomes
\[
\frac{d\rho}{d\tau}=\frac{1}{\beta}\rho(1-\rho).
\]
Separate variables:
\[
\frac{d\rho}{\rho(1-\rho)}=\frac{d\tau}{\beta}.
\]
Using $\frac{1}{\rho(1-\rho)}=\frac{1}{\rho}+\frac{1}{1-\rho}$, integrate to get
\[
\log\Big(\frac{\rho}{1-\rho}\Big)=\frac{\tau}{\beta}+C.
\]
Apply $\rho(0)=\rho_0$ to obtain $C=\log\big(\frac{\rho_0}{1-\rho_0}\big)$. Exponentiating,
\[
\frac{\rho}{1-\rho}=\frac{\rho_0}{1-\rho_0}e^{\tau/\beta}.
\]
Solving this for $\rho$ gives the result.

\textbf{Sqrt-R: $\omega(\rho)=\frac{1}{\rho\sqrt{(1-\rho)}}$.}
We solve the ODE
\[
\frac{d\rho}{d\tau}
= \frac{1}{\beta}\,\rho\sqrt{1-\rho}.
\]
Separate variables:
\[
\frac{d\rho}{\rho\sqrt{1-\rho}}=\frac{1}{\beta}\,d\tau.
\]
Let $u=\sqrt{1-\rho}$, so $\rho=1-u^2$ and $d\rho=-2u\,du$. Then
\[
\int \frac{d\rho}{\rho\sqrt{1-\rho}}
= \int \frac{-2u\,du}{(1-u^2)u}
= -2\int \frac{du}{1-u^2}
= -2\,\operatorname{arctanh}(u).
\]
Hence
\[
\operatorname{arctanh}\big(\sqrt{1-\rho_\tau}\big)
= C - \frac{\tau}{2\beta},
\]
as long as $C\ge \frac{\tau}{2\beta}$. 
Applying $\tanh$ and squaring yields
\[
\sqrt{1-\rho_\tau}=\tanh\left(C-\frac{\tau}{2\beta}\right)\,\,\implies\,\,
\rho_\tau=1-\tanh^2\left(C-\frac{\tau}{2\beta}\right)
=\operatorname{sech}^2\left(C-\frac{\tau}{2\beta}\right).
\]
Since $\rho_{\tau=0}=\rho_0\in(0,1)$, then
$
C=\operatorname{arctanh}\big(\sqrt{1-\rho_0}\big),
$
so the solution is
\[
\rho_\tau
= \operatorname{sech}^2\left(
\operatorname{arctanh}\big(\sqrt{1-\rho_0}\big)
-\frac{\tau}{2\beta}
\right)\,,
\]
as long as $\tau\le \tau_*=2\beta \operatorname{arctanh}\big(\sqrt{1-\rho_0}\big)$. At $\tau=\tau_*$ we have $\rho(\tau_*)=1$. Since the right-hand side of the original ODE vanishes at $\rho=1$ (because $1-\rho=0$), the constant function $\rho_\tau=1$ is a solution for all $\tau\ge \tau_*$. 

\textbf{GRPO: $\omega(\rho)=1/\sqrt{\rho(1-\rho)}$.}
The ODE becomes
\begin{align}
\frac{d\rho}{d\tau}=\frac{1}{\beta}\,\rho^{3/2}\sqrt{1-\rho}.\label{eq:proof grpo effective time}
\end{align}
Define
\[
z:=\sqrt{\frac{1-\rho}{\rho}},
\qquad \text{so that}\qquad \rho=\frac{1}{1+z^2}.
\]
Differentiating the latter gives
\[
\frac{d\rho}{d\tau}=-\frac{2z}{(1+z^2)^2}\frac{d{z}}{d\tau}.
\]
Using Eq. \eqref{eq:proof grpo effective time} and the identity $\rho^{3/2}\sqrt{1-\rho}=\frac{z}{(1+z^2)^2}$, we get for $z>0$ the following linear ODE:
\[
\frac{dz}{d\tau}=-\frac{1}{2\beta}.
\]
Hence
\[
z(\tau)=z_0-\frac{\tau}{2\beta},
\qquad \text{where } z_0=\sqrt{\frac{1-\rho_0}{\rho_0}}.
\]
This formula is valid as long as $z(\tau)\ge 0$, i.e. for $\tau\le \tau_*:=2\beta z_0$.
At $\tau=\tau_*$ we have $z(\tau_*)=0$, hence $\rho(\tau_*)=1$. Since the right-hand side of the original ODE
vanishes at $\rho=1$ (because $1-\rho=0$), the constant function $\rho_\tau\equiv 1$ is a solution for all
$\tau\ge \tau_*$. By uniqueness of solutions for the $z$-equation on the region $z>0$ and the fact that $\rho=1$
is an equilibrium of the $\rho$-equation, the global solution can be written as stated in the lemma.

\textbf{RLOO: $\omega(\rho)=1$.}
The ODE becomes
\[
\frac{d\rho}{d\tau}=\frac{1}{\beta}\rho^2(1-\rho).
\]
Separate variables:
\[
\frac{d\rho}{\rho^2(1-\rho)}=\frac{d\tau}{\beta}.
\]
Using 
$
\frac{1}{\rho^2(1-\rho)}=\frac{1}{\rho^2}+\frac{1}{\rho}+\frac{1}{1-\rho}
$
and integrating gives
\[
-\frac{1}{\rho}+\log\Big(\frac{\rho}{1-\rho}\Big)=\frac{\tau}{\beta}+C,
\]
where, by applying initialization $\rho_{\tau=0}=\rho_0$,
\[
C=-\frac{1}{\rho_0}+\log\Big(\frac{\rho_0}{1-\rho_0}\Big).
\]
Thus, $\rho_\tau$ is determined implicitly by solving
\[
-\frac{1}{\rho_\tau}+\log\Big(\frac{\rho_\tau}{1-\rho_\tau}\Big)=\frac{\tau}{\beta}-\frac{1}{\rho_0}+\log\Big(\frac{\rho_0}{1-\rho_0}\Big)=:s_\tau.
\]
To see uniqueness, note that the left-hand side is a continuous strictly increasing function of $\rho\in(0,1)$.

For an explicit form, set $u_\tau:=\frac{1-\rho_\tau}{\rho_\tau}$ so that $\rho_\tau=\frac{1}{1+u_\tau}$. 
Then, the implicit equation becomes
\[
-(1+u)-\log u = s_\tau
\,\,\,\Longleftrightarrow\,\,\,
u+\log u = -(s_\tau+1).
\]
Exponentiating yields
\[
u e^{u}=e^{-(s_\tau+1)}.
\]
Therefore
\[
u_\tau=W\big(e^{-(s_\tau+1)}\big),
\]
 where $W$ is the principal branch of the Lambert-$W$ function, completing the proof.
\end{proof}

\subsection{Calculating the constants for budget constraints}

In this section, we clarify the constant for a normalization budget on the total weight magnitude applied throughout the learning process that are used to draw Figure \ref{fig:dynamics_comparison_normalized}. To plot the normalized dynamics, we use the closed-form solutions from Propositions~\ref{prop:ode_solutions_regular_time} and~\ref{prop:closed_form_tau} with appropriately rescaled $\beta$ values. This works because normalizing $\omega(\rho)$ by a constant $c$ in the ODE $\frac{d\rho}{dt} = \frac{1}{\beta}\rho(1-\rho)\omega(\rho)$ is equivalent to using the unnormalized weight $\omega$ with $\beta$ replaced by an effective $\beta_\text{eff}=\beta/c$.
For budget-normalized plotting with $\beta=1$, the effective values are: $\beta_{GRPO}=\pi - 2\arcsin(\sqrt{\rho_0})$, $\beta_{Sqrt-R}=2\operatorname{arctanh}(\sqrt{1-\rho_0})$, $\beta_{Linear-R}=\ln(1/\rho_0)$, and $\beta_{RLOO}=1-\rho_0$.

\begin{lemma}
\label{lem:budget}
Consider the dynamics in Eq. \eqref{eq:ODE} with initialization $\rho_0 \in (0,1)$ and target $\rho_*=1$. Under the budget constraint $\int_{\rho_0}^1 \omega(\rho) \, d\rho = 1$, the normalized weight functions for the algorithms are:

\begin{itemize}
    \item GRPO :
    \[
    \omega(\rho) = \frac{1}{\pi - 2\arcsin(\sqrt{\rho_0})} \cdot \frac{1}{\sqrt{\rho(1-\rho)}}
    \]
    
    \item Sqrt-R:
    \[
    \omega(\rho) = \frac{1}{2\operatorname{arctanh}(\sqrt{1-\rho_0})} \cdot \frac{1}{\rho\sqrt{1-\rho}},
    \]
    
    \item Linear-R:
    \[
    \omega(\rho) = \frac{1}{\ln(1/\rho_0)} \cdot \frac{1}{\rho}.
    \]
    
    \item RLOO:
    \[
    \omega(\rho) = \frac{1}{1-\rho_0}.
    \]
\end{itemize}
\end{lemma}

\begin{proof}

\noindent\textbf{GRPO:}
The raw weight is $w(\rho) = \frac{1}{\sqrt{\rho(1-\rho)}}$. The normalization constant $C$ is determined by:
\[
\frac{1}{C} = \int_{\rho_0}^1 \frac{d\rho}{\sqrt{\rho(1-\rho)}} = \left[ 2\arcsin(\sqrt{\rho}) \right]_{\rho_0}^1  = \pi - 2\arcsin(\sqrt{\rho_0}).
\]

\noindent\textbf{Sqrt-R:}
The raw weight is $w(\rho) = \frac{1}{\rho\sqrt{1-\rho}}$. The normalization constant $C$ is determined by:
\[
\frac{1}{C} = \int_{\rho_0}^1 \frac{d\rho}{\rho\sqrt{1-\rho}}.
\]
Using the substitution $u = \sqrt{1-\rho}$, we have $d\rho = -2u \, du$ and $\rho = 1-u^2$. Thus,
\[
\int_{\sqrt{1-\rho_0}}^0 \frac{-2u \, du}{(1-u^2)u} = -2 \int_{\sqrt{1-\rho_0}}^0 \frac{du}{1-u^2} = -\left[2\operatorname{arctanh}(u)\right]_{\sqrt{1-\rho_0}}^0=2\operatorname{arctanh}(\sqrt{1-\rho_0}).
\]

\noindent\textbf{Linear-R:}
The raw weight is $w(\rho) = 1/\rho$. The normalization constant is $1/C = \int_{\rho_0}^1 \frac{d\rho}{\rho} = -\log(\rho_0)$.

\noindent\textbf{RLOO:}
The raw weight is $w(\rho) = 1$. The normalization constant is $1/C = 1 - \rho_0$.
\end{proof}

\subsection{Proof of Proposition~\ref{prop:single}}
The time required to reach $\rho_t=\rho_*$ starting from $\rho_0$ is
\begin{align}\label{eq:hitting tim}
T(\rho_0,\rho_*;\omega)
\;=\;
\int_0^T dt
\;=\;
\int_{\rho_0}^{\rho_*}
\frac{d\rho}{d\rho_t/dt}\,\,.
\end{align}
Thus, using the dynamics Eq. \eqref{eq:ODE}, the optimization problem becomes
\[
\min_{\omega:[0,1]\rightarrow\R_{\geq0}} \;
 \int_{\rho_0}^{\rho_*}
\frac{d\rho}{\rho(1-\rho) \cdot\omega(\rho)}
\quad
\text{s.t. }
\int_{\rho_0}^{\rho_*} \omega(\rho)\,d\rho \leq 1\,.
\]
Define $a(\rho)=\frac{1}{\rho(1-\rho)}$.
By Cauchy--Schwarz,
\[
\left(\int_{\rho_0}^{\rho_*} \sqrt{a(\rho)}\,d\rho\right)^2
\;\le\;
\left(\int_{\rho_0}^{\rho_*}
\frac{a(\rho)}{\omega(\rho)}\,d\rho\right)
\left(\int_{\rho_0}^{\rho_*} \omega(\rho)\,d\rho\right).
\]
Since the second factor is at most $1$, we obtain the lower bound
$
T(\rho_0,\rho_*;\omega)
\;\ge\;
\beta
\left(\int_{\rho_0}^{\rho_*}
\frac{d\rho}{\sqrt{\rho(1-\rho)}}\right)^2, 
$
with 
equality if and only if
$
\omega(\rho) \;\propto\; \sqrt{a(\rho)}\,
$ and the budget constraint is tight. 
Therefore, the optimal choice is
\[
\omega_{\mathrm{opt}}(\rho)
=
c\,\frac{1}{\sqrt{\rho(1-\rho)}}
\quad \text{for } \rho\in[\rho_0,\rho_*]\,,
\]
with
$c$ chosen to saturate the budget normalization constraint. Concretely, this gives 
$$\omega_{\text{opt}}(\rho) = \frac{1}{2 \left( \arcsin\sqrt{\rho_*} - \arcsin\sqrt{\rho_0} \right)}\cdot\frac{1}{\sqrt{\rho(1-\rho)}}\,.$$

\subsection{Proof of Proposition~\ref{prop:single_effective_time}}
Recall the population dynamics in Eq. \eqref{eq:ODE} and the effective time-to-target definition:
\[
T_g(\rho_0,\rho_*;\omega)
:=
\int_{\rho_0}^{T} g(\rho_t)\,dt,
\qquad \text{where } \rho_T=\rho_*,
\]
with $g(\rho)=1/\rho$.
Using the change of variables $dt = d\rho / \dt{ \rho_t}$ yields
\begin{align}\label{eq:effective_hitting_time}
T_g(\rho_0,\rho_*;\omega)
&=
\int_{\rho_0}^{\rho_*} \frac{g(\rho)}{d\rho_t/dt}\,d\rho
=
\int_{\rho_0}^{\rho_*}
\frac{g(\rho)}{\frac{1}{\beta}\rho(1-\rho)\omega(\rho)}\,d\rho=
\beta\int_{\rho_0}^{\rho_*}
\frac{d\rho}{\rho^2(1-\rho)\,\omega(\rho)}.
\end{align}
Thus, the optimization of minimizing $T_g$ subject to $\omega(\rho)\ge 0$ and the budget constraint
$\int_{\rho_0}^{\rho_*}\omega(\rho)\,d\rho\le 1$ becomes
\[
\min_{\omega:[0,1]\rightarrow\R_{\geq0}}\;
\int_{\rho_0}^{\rho_*}
\frac{a(\rho)}{\omega(\rho)}\,d\rho
\quad\text{s.t.}\quad
\int_{\rho_0}^{\rho_*}\omega(\rho)\,d\rho\le 1.
\]
where we defined $a(\rho):=\frac{1}{\rho^2(1-\rho)}$. 
By Cauchy--Schwarz,
\[
\left(\int_{\rho_0}^{\rho_*}\sqrt{a(\rho)}\,\mathrm{d}\rho\right)^2
\le
\left(\int_{\rho_0}^{\rho_*}\frac{a(\rho)}{\omega(\rho)}\,\mathrm{d}\rho\right)
\left(\int_{\rho_0}^{\rho_*}\omega(\rho)\,\mathrm{d}\rho\right).
\]
Since the second factor is at most $1$, we obtain the lower bound
\[
T_g(\rho_0,\rho_*;\omega)
\ge
\beta\left(\int_{\rho_0}^{\rho_*}\sqrt{\frac{1}{\rho^2(1-\rho)}}\,\mathrm{d}\rho\right)^2,
\]
with equality if and only if
\[
\omega(\rho)\ \propto\ \sqrt{a(\rho)}
\ =\ \frac{1}{\rho\sqrt{1-\rho}}
\quad\text{for }\rho\in[\rho_0,\rho_*]
\]
with the proportionality constant determined to saturate the budget constraint.

\subsection{Origins of Linear-R}\label{sec:onLR}

We arrive at Linear-R by selecting asymmetric weighting $\omega(\rho)=1/\rho$ in the general form of RLVR algorithms in Equation~\eqref{eq:weighted emp}.  We specifically choose such a weight that aggressively focuses on small $\rho$ to encourage larger gradient signal. 

We can alternatively arrive at the same weighting by applying the forward-engineering recipe \citep{thrampoulidis_AdvantageShapingSurrogate_2025} to a logarithmic surrogate reward. Concretely, start from a surrogate $F(u)=\log(u)$ and optimize $F(\rho_x(\theta))$: direct differentiation yields $F'(\rho_x(\theta))\cdot \nabla_\theta\rho_x(\theta)=\frac{1}{\rho_x(\theta)}\cdot \nabla_\theta\rho_x(\theta)$, which coincides with our Linear-R update after using an RLOO estimator for $\nabla_\theta\rho_x(\theta)$. 

It is worth noting that another popular RL algorithm, which is different to our Linear-R, has also recently been shown to map to the logarithmic surrogate reward by \citet{davis_WhatObjectiveReasoning_2025}.
In our notation, rejection sampling updates the parameters in the direction
\begin{align}\label{eq:rejection sampling}
\frac{1}{M^+}\sum_{i=1}^{M^+}\nabla_\theta \log \pi_\theta(y|x)  = \widehat\nabla_1
\end{align}
In population limit, this equals  
\[
\E_{y\sim\picond}[\nabla_\theta \log \pi_\theta(y|x) | r(y)=1] = \frac{\E[r(y) \nabla_\theta \log\pi_\theta(y|x)]}{\rhotheta} = \frac{1}{\rhotheta} \nabla_\theta \rhotheta=\frac{\mathrm{d}\log \rho_\theta}{\mathrm{d}\rho_\theta}\,,
\]
which corresponds to asymptotically ($M\gg1$) maximizing the log surrogate reward. 

This can also be seen as a special instance of Equation~\eqref{eq:weighted emp} with $\omega_x(\rho)=1/\rho$ and $\hat d_x(\theta) = \widehat\nabla_1$. That is, rejection sampling uses a $1/\rho$ weight as Linear-R, but uses the REINFORCE estimate for the gradient rather than the RLOO used for Linear-R. 

We also run this rejection sampling weighting for the TinyZero experiment, and it performed significantly worse than the other considered algorithms (Plateau-R, Linear-R, Uniform-R, GRPO, RLOO).

For completeness, we mention that our Sqrt-R can be analogously interpreted as optimizing the surrogate reward $F(u)=-2\operatorname{arctanh}(\sqrt{1-u})$ using again the RLOO gradient estimator. To see this, recall from Sec. \ref{sec:assymetric} that $F'(\rho)=\frac{1}{\rho\sqrt{1-\rho}}$, which is exactly Sqrt-R's weight.

\end{document}